\documentclass{article}

\usepackage{fullpage}
\usepackage{parskip}

\usepackage{textcomp}
\usepackage{bm}
\usepackage{dsfont}
\usepackage{newtxtext}

\usepackage[utf8]{inputenc}
\usepackage[T1]{fontenc}    
\usepackage{booktabs}       
\usepackage{nicefrac}       

\usepackage{etoolbox}
\let\etbforlistloop\forlistloop
\cslet{blx@noerroretextools}\empty%
\usepackage[date=year,eprint=false,doi=false,isbn=false,maxcitenames=2,natbib,url=false,sorting=nyt,sortcites=true,style=trad-abbrv,backref=true]{biblatex}

\DeclareSourcemap{
    \maps[datatype=bibtex]{
        \map{
            \step[fieldset=editor, null]
            \step[fieldset=series, null]
            \step[fieldset=language, null]
            \step[fieldset=address, null]
            \step[fieldset=location, null]
            \step[fieldset=month, null]
            \step[fieldset=annotated, null]
        }
    }
}
\usepackage{amsmath,amssymb,amsthm,mathtools}
\usepackage{thmtools}
\usepackage{thm-restate}
\usepackage{microtype}
\usepackage{algorithm}
\usepackage[noend]{algpseudocode}
\usepackage{enumitem}
\usepackage{CJKutf8}
\usepackage[svgnames]{xcolor}
\usepackage[colorlinks=true,citecolor=Navy,linkcolor=Maroon,urlcolor=Orchid,bookmarksnumbered,hypertexnames=false,pdfdisplaydoctitle,pdfusetitle,unicode]{hyperref}
\usepackage{todonotes}
\usepackage{ulem}
 
\theoremstyle{plain}
\newtheorem{theorem}{Theorem}
\newtheorem{lemma}{Lemma}
\newtheorem{proposition}{Proposition}

\theoremstyle{definition}

\newtheorem{assumption}{Assumption}
\newtheorem{remark}{Remark}

\newcommand{\R}{\mathbb{R}}
\newcommand{\Z}{\mathcal{Z}}
\newcommand{\X}{\mathcal{X}}
\newcommand{\VV}{\mathcal{V}}
\newcommand{\HH}{\mathcal{H}}
\newcommand{\argmin}{\operatorname*{arg\,min}}
\newcommand{\argmax}{\operatorname*{arg\,max}}

\newcommand{\rank}{\operatorname{rank}}
\newcommand{\tr}{\operatorname{tr}}
\newcommand{\Id}{\mathrm{I}}
\newcommand{\ip}[2]{\langle #1, #2\rangle}

\newcommand{\normA}[2]{\norm*{#1}_{#2}}
\newcommand{\normAsq}[2]{\norm*{#1}_{#2}^{2}}

\usepackage{upref}
\usepackage[capitalize,noabbrev]{cleveref}
\usepackage{autonum}
\let\forlistloop\etbforlistloop

\crefname{step}{Step}{Steps}
\Crefname{step}{Step}{Steps}
\crefname{assumption}{Assumption}{Assumptions}
\Crefname{assumption}{Assumption}{Assumptions}
\crefname{appendix}{Appendix}{Appendices}
\Crefname{appendix}{Appendix}{Appendices}
\pretocmd{\appendix}{\crefalias{section}{appendix}}{}{}
\let\ip\relax
\DeclarePairedDelimiterX{\ip}[2]{\langle}{\rangle}{#1,\,#2}
\let\prn\relax
\DeclarePairedDelimiter{\prn}{(}{)}
\let\set\relax
\DeclarePairedDelimiter{\set}{\{}{\}}
\let\Set\relax
\DeclarePairedDelimiterX{\Set}[2]{\{}{\}}{\,{#1}\,:\,{#2}\,}

\let\Brc\relax
\DeclarePairedDelimiterX{\Brc}[2]{[}{]}{\,{#1}\,\middle|\,{#2}\,}

\let\norm\relax
\DeclarePairedDelimiter{\norm}{\|}{\|}

\newcommand{\savednormallabel}{}
\AtBeginDocument{\global\let\savednormallabel\label}
\newcommand{\restorenormallabel}{\global\let\label\savednormallabel}
\makeatletter
\newcommand{\wraprestatablewithlabelrestore}[1]{%
  \csletcs{restatable@orig@#1}{#1}%
  \csdef{#1}{%
    \@ifstar{%
      \csuse{restatable@orig@#1}*%
      \restorenormallabel
    }{%
      \csuse{restatable@orig@#1}%
    }%
  }%
}

\title{Simple Projection-Free Algorithm for Contextual Recommendation with Logarithmic Regret and Robustness}

\author{%
Shinsaku Sakaue%
\thanks{CyberAgent, Tokyo, Japan. National Institute of Informatics, Tokyo, Japan. Center for Advanced Intelligence Project, RIKEN, Tokyo, Japan. Email: shinsaku.sakaue@gmail.com.}%
}

\date{}

\begin{document}
\maketitle
\begin{abstract}
Contextual recommendation is a variant of contextual linear bandits in which the learner observes an (optimal) action rather than a reward scalar.
Recently, Sakaue et al.~(2025) developed an efficient Online Newton Step (ONS) approach with an $O(d\log T)$ regret bound, where $d$ is the dimension of the action space and $T$ is the time horizon.
In this paper, we present a simple algorithm that is more efficient than the ONS-based method while achieving the same regret guarantee.
Our core idea is to exploit the improperness inherent in contextual recommendation, leading to an update rule akin to the second-order perceptron from online classification.
This removes the Mahalanobis projection step required by ONS, which is often a major computational bottleneck.
More importantly, the same algorithm remains robust to possibly suboptimal action feedback, whereas the prior ONS-based method required running multiple ONS learners with different learning rates for this extension.
We describe how our method works in general Hilbert spaces (e.g., via kernelization), where eliminating Mahalanobis projections becomes even more beneficial.
\end{abstract}

\section{Introduction}
A pervasive challenge in data-driven decision-making is learning from observed \emph{actions} rather than explicit utility signals: a system records what a user or expert \emph{chooses}---a purchased item, a selected route, a clinical treatment---without directly observing the hidden preference that motivated the choice.
This action-only feedback regime arises in various fields, e.g., inverse optimization~\citep{Ahuja2001-cv,Chan2025-bl}, inverse reinforcement learning~\citep{Ng2000-sf}, and learning from revealed preference~\citep{Birge2022-fq}.
\emph{Contextual recommendation} (a.k.a.~\emph{online inverse linear optimization}) is a fundamental online learning problem that formalizes such action-only feedback scenarios, in which the learner sequentially recommends actions and infers a user's hidden preference from observed user actions.
This problem is a variant of (contextual) linear bandits \citep{Dani2008-uo,Abbasi-yadkori2011-st}---in the standard setting, the learner observes a scalar reward, whereas here the learner observes an (optimal) action---and has been extensively studied \citep{Barmann2017-wl,Barmann2020-hh,Gollapudi2021-ad,Besbes2021-ak,Besbes2025-ck,Sakaue2025-fe,Sakaue2025-vb,Oki2026-kc}.

Specifically, let $\VV$ be a real Hilbert space with inner product $\ip{\cdot}{\cdot}$ 
(one may think of $\VV=\R^d$ with the standard dot product as a canonical example).
At each round $t$, the learner observes a nonempty feasible set $\X_t\subseteq\VV$, recommends an action
$\hat x_t\in\argmax_{x\in\X_t}\ip{\hat w_t}{x}$ for some prediction $\hat w_t\in\VV$, and then observes a user action $x_t\in\X_t$, which is optimal (or approximately optimal) with respect to the user's unknown preference vector $u\in\VV$.
The learner's performance is measured by the regret, 
$R_T(u)\coloneqq\sum_{t=1}^T\ip{u}{x_t-\hat x_t}$, which quantifies the cumulative gap between the utilities of the user's actions and the learner's recommendations over $T$ rounds. 
In what follows, we refer to $\hat x_t -x_t$ as a \emph{residual} at round $t$; 
the regret is the inner product of $u$ with the negative cumulative residual.

\citet{Sakaue2025-vb} recently proposed an algorithm based on the Online Newton Step (ONS) with the state-of-the-art regret bound of $O(d\log T)$ under $\VV=\R^d$; 
this is the first \emph{efficient} logarithmic-regret algorithm in the sense that the per-round time complexity is independent of $T$.
However, an ONS update involves a \emph{Mahalanobis projection} step at each round. 
This is often the computational bottleneck, especially in high-dimensional or kernelized settings.
Specifically, when $\VV=\R^d$ and the learner's domain is the unit ball, the projection step solves
$\min_{w: \norm{w}_2\le 1}(w'-w)^\top A (w'-w)$
for some $w'\in\R^d$ and a positive-definite matrix $A\in\R^{d\times d}$.
Even in this unit-ball case, the exact Mahalanobis projection requires at least $\Omega(d^\omega)$ arithmetic operations, where $\omega \in [2,3]$ is the matrix multiplication exponent, while standard exact implementations typically take $O(d^3)$ time; see \citet{Wang2025-lj} for a detailed discussion.
The ONS-based approach of \citet{Sakaue2025-vb} involves such a projection step at every round, contributing a substantial computational overhead.
Therefore, the core question we address is: \emph{can we avoid Mahalanobis projections while retaining logarithmic regret guarantees?}\looseness=-1

\begin{table}[t]
  \centering
  \caption{Total time complexity over $T$ rounds under $\VV=\R^d$. 
    For the ONS-based methods, the learner's domain is the unit ball.
    Let $\tau_{\mathrm{opt}}(t)$ be the time to solve $\hat x_t\in\argmax_{x\in\X_t}\ip{\hat w_t}{x}$ and $\omega \in [2,3]$ the matrix multiplication exponent.   
  All methods enjoy $R_T(u)=O(d\log T)$ regret bounds under optimal action feedback. For details of these complexity bounds, see \cref{sec:implementation} and \cref{app:ons-comparison}.}
  \medskip
  \begin{tabular}{ll}
    \toprule
    Method & Total time complexity \\
    \midrule
    ONS-based method
    \citep{Sakaue2025-vb}
    &
    $O\prn*{T d^\omega + \sum_{t=1}^T \tau_{\mathrm{opt}}(t)}$
    \\
    \midrule
    ONS-based method \citep{Sakaue2025-vb} with LightONS \citep{Wang2025-lj}
    &
    $O\prn*{T d^2 + d^\omega \sqrt{T\log T} + \sum_{t=1}^T \tau_{\mathrm{opt}}(t)}$
    \\
    \midrule
    CoRectron (\cref{alg:second-order})
    &
    $O\prn*{T d^2 + \sum_{t=1}^T \tau_{\mathrm{opt}}(t)}$
    \\    \bottomrule
  \end{tabular}
  \label{tab:finite-dim-cost-comparison}
\end{table}

\subsection{Our contributions}
We summarize our contributions at a high level and then outline our core idea and analysis techniques.
\begin{enumerate}[leftmargin=*]
  \item \textbf{Simple projection-free algorithm.} We present a simple and efficient algorithm (\cref{alg:second-order}) that avoids Mahalanobis projections by exploiting the improperness arising in contextual recommendation. 
  As summarized in \cref{tab:finite-dim-cost-comparison}, our algorithm has lower total time complexity than the previous ONS-based method \citep{Sakaue2025-vb} and the one accelerated by LightONS, a recent fast ONS variant \citep{Wang2025-lj}.
  \item \textbf{Logarithmic regret bound.} We prove a regret bound controlled by a log-determinant potential (\cref{thm:main}), which yields the same bound of $O(d\log T)$ as \citet{Sakaue2025-vb} when $\VV=\R^d$.
  \item \textbf{Robustness to suboptimal feedback.} The same algorithm achieves a regret bound that scales with the square root of the cumulative suboptimality of feedback actions (\cref{thm:robust-demo}).
  \item \textbf{Contextual models.} We study the problem in a Hilbert space and show how contextual models can be handled in a unified manner, where eliminating Mahalanobis projections is even more beneficial.
  \item \textbf{Empirical support.} Experiments in \cref{app:dashboard-layout-preview} show that our algorithm is faster and achieves smaller regret than the ONS-based method.
  Moreover, it is stable across hyperparameter choices.
\end{enumerate}

\textbf{Our core idea: exploiting improperness.\;}
In contextual recommendation, the learner recommends actions $\hat x_t \in \argmax_{x\in\X_t}\ip{\hat w_t}{x}$ based on predicted utility vectors $\hat w_t \in \VV$.
Therefore, positive rescaling of $\hat w_t$ leaves the recommendation unchanged.
This leads to an ``improper'' viewpoint: we do not need any scale constraint on $\hat w_t$; in particular, $\|\hat w_t\|$ is allowed to exceed $\|u\|$, the (unknown) scale of the user's preference vector.
Building on this observation, our algorithm performs a projection-free second-order update built from the past residuals.
This update is reminiscent of the classical second-order perceptron~\citep{Cesa-Bianchi2005-on} from online classification.
Motivated by this connection, we name our method CoRectron (\uline{co}ntextual \uline{rec}ommendation via second-order percep\uline{tron}) for ease of reference.

\textbf{Regret analysis.\;}
Our analysis rests on a one-sided optimality inequality, $\ip{\hat w_t}{\hat x_t-x_t}\ge 0$, that follows from the optimality of $\hat x_t$ for $\hat w_t$.
This yields a sign condition that enables us to control the growth of a potential of the cumulative residuals (\cref{lem:potential-growth}). 
Combined with the elliptical potential lemma, this leads to a template regret bound controlled by the log-determinant of the Gram matrix of residuals (\cref{thm:main}), which gives an $O(d\log T)$ regret bound in the case of $\VV=\R^d$.
Interestingly, this clean logarithmic regret guarantee contrasts with existing mistake bounds of the second-order perceptron for online classification, which involve margin-dependent terms that may be arbitrarily large in the worst case \citep{Cesa-Bianchi2005-on,Orabona2012-xt}.
We elaborate on this difference in \cref{sec:algorithm}.

\textbf{Robustness.\;}
Importantly, our template bound \eqref{eq:main-bound-optimized} continues to hold even when the observed actions can be suboptimal. 
Consequently, without changing the algorithm, we can establish a regret bound that scales with the square root of the cumulative suboptimality of the feedback actions (\cref{thm:robust-demo}). 
This result contrasts with the approach of \citet{Sakaue2025-vb} for obtaining the same robustness guarantee: they use MetaGrad \citep{van-Erven2021-ji}, which runs $O(\log T)$ ONS learners with different learning rates to adapt to an unknown suboptimality level, thereby increasing the computational cost by a factor of $O(\log T)$.

\textbf{Organization.\;}
The rest of this paper is organized as follows.
\Cref{sec:related-work} overviews related work.
\cref{sec:setup} formalizes the problem.
\cref{sec:method} presents CoRectron and the template regret bound.
\cref{sec:robust} provides a regret bound for suboptimal feedback actions.
\cref{sec:lifting} describes a unified view for handling contextual models. 
\cref{sec:implementation} discusses implementation and major computational factors.

\subsection{Related work}\label{sec:related-work}
\Citet{Barmann2017-wl,Barmann2020-hh} introduced the problem setting of contextual recommendation (as an approach to inverse optimization) and proposed a method based on online convex optimization (OCO) with an $O(\sqrt{T})$ regret bound.
Later, cutting-plane-based methods improved the dependence on $T$: 
\citet{Besbes2021-ak,Besbes2025-ck} gave an $O(d^4\log T)$ regret bound, and \citet{Gollapudi2021-ad} established $O(d \log T)$ and $\exp(O(d\log d))$ regret bounds.
These cutting-plane methods require per-round time polynomial in~$T$, whereas \citet{Sakaue2025-vb} achieved an $O(d\log T)$ regret bound with per-round time independent of~$T$ via ONS. 
Under additional structural assumptions, finite regret bounds smaller than $\exp(O(d\log d))$ were also obtained \citep{Sakaue2025-fe,Oki2026-kc}, while our focus is on the general setting without such assumptions.

ONS is a well-established OCO algorithm \citep{Hazan2007-ta}, and its Mahalanobis projection step is recognized as a major computational bottleneck, leading to a line of work on faster online exp-concave optimization \citep{Luo2016-sv,Mhammedi2019-iw,Mhammedi2023-hi,Wang2025-lj}. 
In particular, \citet{Wang2025-lj} recently proposed LightONS, an efficient ONS variant that reduces projection overhead; however, the projection step still constitutes a non-negligible cost. 
Thus, although the method of \citet{Sakaue2025-vb} can be accelerated by plugging in LightONS, the resulting method still has larger worst-case complexity and hence is not as efficient as our projection-free method, as summarized in \cref{tab:finite-dim-cost-comparison}. 
Moreover, our method remains robust to suboptimal feedback and extends to general Hilbert spaces; these are non-trivial extensions not covered by the ONS approach.

Kernelized second-order online learning has also been widely studied, including sketching/embedding methods for acceleration \citep{Calandriello2017-tj,Calandriello2017-cu}. These approaches typically involve an efficiency--regret trade-off. Hence, even when such tools are plugged into the ONS approach for contextual recommendation, they generally improve runtime at the price of weaker guarantees. Similar acceleration techniques may also be useful for our method, and exploring this is an interesting direction for future work.

For projection-free exp-concave OCO with a linear-optimization oracle, the best regret bound is $\tilde O(d^{2/3}T^{2/3})$ \citep{Garber2023-mr}, where the tilde hides logarithmic factors; 
more generally, projection-free OCO with a separation oracle can achieve $\tilde O(\sqrt{dT}+d^2)$~\citep{Mhammedi2025-kv}. 
These are much larger than logarithmic regret, and hence our $O(d\log T)$ bound is specific to the improper structure of contextual recommendation.

Finally, our perspective is different from unconstrained/parameter-free OCO \citep{Mcmahan2012-tf,Cutkosky2018-cy,Jacobsen2023-zg,Cutkosky2024-vv}.
That line of work aims to compete with comparators of unknown scale with minimal regret degradation from that achievable when the comparator scale is known.
By contrast, we exploit the improperness induced by scale invariance in contextual recommendation: we may impose any normalizing assumption on the scale of the user preference (e.g., $\norm{u}\le 1$) while allowing the learner's predictions to exceed that scale, and we use this asymmetry to reduce the computation cost without worsening regret guarantees.

%

\section{Problem setting and background}\label{sec:setup}
Let $\VV$ be a real Hilbert space with inner product $\ip{\cdot}{\cdot}$ and induced norm $\norm{\cdot}$.
For rounds $t=1,\dots,T$, a user arrives with contextual information encoded as a feasible action set $\X_t \subseteq \VV$, which is nonempty and weakly compact.\footnote{
  This standard condition ensures that the maximizer set $\argmax_{x\in\X_t}\ip{w}{x}$ is nonempty for all $w\in\VV$.
}
The user takes an action $x_t\in \X_t$, and its utility is given by $\ip{u}{x_t}$, where $u\in \VV$ represents the user's hidden preference and defines the linear functional $x\mapsto \ip{u}{x}$. 
This $u$ is unknown and fixed across rounds.
Unless otherwise stated, we assume that $x_t$ is optimal for $u$ over~$\X_t$, as formalized below.
\begin{assumption}[Optimal-action feedback]\label{ass:optimal-feedback} 
  It holds that $x_t\in\argmax_{{x\in\X_t}} \ip{u}{x}$ for every $t$.
\end{assumption}

At each round $t$, the learner observes $\X_t$ and, before observing $x_t$, computes a utility-vector prediction $\hat w_t\in \VV$ and recommends an action $\hat x_t\in \argmax_{x\in \X_t}\ip{\hat w_t}{x}$; 
here, the learner has access to a linear-optimization oracle over $\X_t$.
Then, the learner observes the user's action $x_t$ (but not $u$ nor the realized utility value $\ip{u}{x_t}$).
The learner's performance is measured by the \emph{regret}, 
  $R_T(u)\coloneqq\sum_{t=1}^T\ip{u}{x_t-\hat x_t}$,
which quantifies the cumulative suboptimality of the learner's recommendations compared to the user's actions $x_t$.
We use $r_t\coloneqq\ip{u}{x_t-\hat x_t}$ to denote the instantaneous regret at round $t$.
Throughout this paper, we make the following standard boundedness assumption. 
\begin{assumption}\label{ass:bounded-payoff}
There exists $B>0$ such that, for all $t$ and all $x,x'\in\X_t$, $\ip{u}{x-x'}\le B$ holds.
\end{assumption}
In particular, the instantaneous regret satisfies $r_t \in[-B,B]$.
Under \cref{ass:optimal-feedback}, $r_t\in[0,B]$ holds.

\textbf{Hilbert-space operator notation.\;}
We allow $\VV$ to be infinite-dimensional.
We use $\Id$ to denote the identity operator on $\VV$ and $\Id_m$ to denote the identity matrix on $\R^m$. 
For $a\in\VV$, define $a\otimes a\colon\VV\to\VV$ by
$(a\otimes a)(v)=\ip{a}{v}\,a$ for all $v\in\VV$; in finite dimensions this coincides with $aa^\top$.
For a bounded linear map $G\colon\R^m\to\VV$, let $G^\ast\colon\VV\to\R^m$ denote its adjoint,
defined by $\ip{Ga}{v}=a^\top(G^\ast v)$ for all $a\in\R^m$ and $v\in\VV$ (transpose in finite dimensions).
For $\lambda>0$ and $a_1,\dots,a_m\in\VV$, let
$A\colon\VV\to\VV$ be a linear operator of the form $A=\lambda\Id+\sum_{i=1}^m a_i\otimes a_i$.
Then, for all $v\in\VV$, we have $\ip{v}{A v}=\lambda\|v\|^2+\sum_{i=1}^m\ip{a_i}{v}^2\ge\lambda\|v\|^2$.
Hence $A\succeq \lambda\Id$ holds and, in particular, $A$ is invertible by the Lax--Milgram theorem \citep{Lax1954-vh}. 
See \cref{app:operator-preliminaries} for additional details on Hilbert-space operators. 

\textbf{Online learning approach.\;}
We briefly review the online learning approach taken in prior work, which is relevant to our idea. 
\Citet{Barmann2017-wl,Barmann2020-hh} proposed to use $\ell_t\colon w \mapsto \ip{w}{\hat x_t - x_t}$ as loss functions and compute predictions $\hat w_t$ via online convex optimization (OCO) on a bounded domain $\mathcal{W} \subseteq \R^d$ that contains $u$. 
Then, as shown below, the standard OCO regret bound for $\ell_t$ against the comparator $u$ serves as an upper bound on the contextual recommendation regret $R_T(u)$:\looseness=-1 
\[
  R_T(u) = \sum_{t=1}^T \ip{u}{x_t - \hat x_t} \le \sum_{t=1}^T \ip{u}{x_t - \hat x_t} + \sum_{t=1}^T \ip{\hat w_t}{\hat x_t - x_t} = \sum_{t=1}^T \ell_t(\hat w_t) - \sum_{t=1}^T \ell_t(u),
\]
where the first inequality follows from the optimality of $\hat x_t$ for $\hat w_t$ over $\X_t$.\footnote{
  By the above argument, the OCO-based approach also yields an upper bound on $\sum_{t=1}^T \ip{\hat w_t}{\hat x_t - x_t}$, the online counterpart of the cumulative \emph{suboptimality loss} $w \mapsto \max_{x\in\X_t}\ip{w}{x}-\ip{w}{x_t}$, which is a common loss function in inverse optimization \citep{Mohajerin-Esfahani2018-jf}.
  As an evaluation metric, however, this quantity is often less informative because the learner can make it arbitrarily small by shrinking the scale of $\hat w_t$ (although informative to some extent when the learner's domain is bounded away from the origin \citep{Sakaue2025-vb}). 
  Thus, our work focuses on bounding the regret $R_T(u)$.
} 
Therefore, simply applying a standard OCO algorithm (e.g., the online gradient descent) to the losses $\ell_t$ yields an $O(\sqrt{T})$ regret bound for contextual recommendation. 
\Citet{Sakaue2025-vb} improved this to $O(d\log T)$ by applying ONS to modified exp-concave losses, but the high-level idea of using OCO on a bounded domain remains the same. 
Accordingly, prior OCO-based approaches require projections to ensure $\hat w_t\in\mathcal{W}$.

\textbf{Improperness induced by scale invariance.\;}
In contextual recommendation, the learner's recommendation is obtained as $\hat x_t \in \argmax_{x\in\X_t}\ip{\hat w_t}{x}$.
Thus, the induced actions $\hat x_t$
are invariant under positive rescaling of the corresponding utility vectors $\hat w_t$.
Since the regret $R_T(u) = \sum_{t=1}^T \ip{u}{x_t - \hat x_t}$ is evaluated on the recommended actions $\hat x_t$ (rather than on $\hat w_t$), multiplying $\hat w_t$ by any positive scalar leaves $\hat x_t$---and therefore $R_T(u)$---unchanged.
Also, the observed user actions $x_t$ carry information only about the direction of $u$, and the learner cannot hope to infer the scale $\|u\|$ from feedback alone.
In light of the above, we may fix the scale of $u$ (e.g., $\|u\|=1$) while allowing the learner to output utility vectors $\hat w_t$ of arbitrary scale.
This \emph{improper} viewpoint is central to eliminating projections on the learner's iterates. 
The algorithm and analysis that follow do not require prior knowledge of $\|u\|$.

\begin{algorithm}[t]
\caption{CoRectron: contextual recommendation via second-order perceptron}\label{alg:second-order}
\begin{algorithmic}[1]
  \Require Regularization parameter $\lambda>0$, $A_0 = \lambda\Id$, $\zeta_0 = 0 \in\VV$
\For{$t=1,2,\dots,T$}
  \State Compute $\hat w_{t}\gets -A_{t-1}^{-1}\zeta_{t-1}$
  \State Observe context $\X_t\subseteq\VV$
  \State Recommend $\hat x_t\in\argmax_{x\in\X_t}\ip{\hat w_t}{x}$
  \State Observe $x_t\in\X_t$ and set $g_t\gets \hat x_t-x_t$
  \State Update $A_t\gets A_{t-1}+g_t\otimes g_t$ and $\zeta_t\gets \zeta_{t-1}+g_t$
\EndFor
\end{algorithmic}
\end{algorithm}

\section{CoRectron}\label{sec:method}
This section presents our algorithm and a template regret bound, which we instantiate in \cref{sec:lifting}.

\subsection{Algorithm description}\label{sec:algorithm}
For the recommended actions $\hat x_t$ and the observed actions $x_t$, we define the \emph{residual} at round $t$ by 
$g_t \coloneqq \hat x_t - x_t$,\footnote{
  The residual $g_t$ coincides with a subgradient of the suboptimality loss at $\hat w_t$.
}
and we use $\zeta_t \coloneqq \sum_{s=1}^t g_s$ to denote the cumulative residual up to round $t$.
Then, 
the regret is expressed as $R_T(u) = -\ip{u}{\zeta_T}$.
Our algorithm, CoRectron (\cref{alg:second-order}), maintains the cumulative residual $\zeta_t$ and the second-order preconditioner $A_t$. 
The utility-vector prediction for round $t$ is computed as $\hat w_{t} = -A_{t-1}^{-1}\zeta_{t-1}$. 
Importantly, this algorithm is projection-free.

\textbf{Relation to second-order perceptron.\;}
\cref{alg:second-order} resembles the second-order perceptron for online classification \citep{Cesa-Bianchi2005-on}.
In that setting, the context information is $\X_t=\{ y z_t / 2 \,\colon\, y \in \{ -1, +1 \} \}$ for the $t$-th input vector $z_t$. 
A label prediction $\hat y_t = \mathrm{sign}(\ip{\hat w_t}{z_t}) \in \{ -1, +1 \}$ induces the recommended action $\hat x_t = \hat y_t z_t / 2$, and the true label $y_t \in \{ -1, +1 \}$ reveals the correct action $x_t = y_t z_t / 2$ as feedback. 
The algorithm performs the second-order update $\hat w_{t+1} = -A_{t}^{-1} \zeta_{t}$ only when a mistake occurs, i.e., $g_t = \hat x_t - x_t = 0$ if $\hat y_t = y_t$ and $g_t = -y_t z_t$ otherwise. 
Our CoRectron can be seen as an extension of this classification algorithm to the contextual recommendation setting, where the context $\X_t$ is a set of feasible actions rather than a single input vector with two possible labels. 
Importantly, our analysis for contextual recommendation yields clean logarithmic bounds on the regret $R_T(u)$, as we will see below. 
This is in contrast to the classification setting, where mistake bounds of the second-order perceptron involve margin-dependent terms that can be arbitrarily large in the worst case \citep{Cesa-Bianchi2005-on,Orabona2012-xt}. 
In this sense, our analysis sheds light on the power of the second-order perceptron in contextual recommendation.

\subsection{Template regret bound}\label{sec:analysis}
We analyze the regret of \cref{alg:second-order}. 
For later use, define the Gram matrix of the residuals as
\begin{equation}\label{eq:gram-template}
  K_t \coloneqq \prn*{\ip{g_i}{g_j}}_{i,j=1}^t \in \R^{t\times t}
  \qquad\text{for $t = 1, \dots, T$}.
\end{equation} 
We also define the \emph{cumulative potential} as
\begin{equation}\label{eq:potential}
  \Phi_t \coloneqq \ip{\zeta_t}{A_t^{-1}\zeta_t}=\normAsq{\zeta_t}{A_t^{-1}} 
  \qquad\text{for $t = 1, \dots, T$}.
\end{equation}
This can be seen as a variant of the well-known \emph{elliptical potential} $\sum_{s=1}^t \ip{g_s}{A_s^{-1}g_s}$, where we use the cumulative residual $\zeta_t$ instead of taking the sum of the per-round quadratic forms $\ip{g_s}{A_s^{-1}g_s}$.

The goal of this subsection is to prove the following template regret bound.
\begin{theorem}[Main regret bound]\label{thm:main}
For any $u\in\VV$, \cref{alg:second-order} with $\lambda>0$ achieves
\begin{equation}\label{eq:main-bound-optimized}
  R_T(u)\le \normA{u}{A_T}\sqrt{\log\det\prn*{\Id_T+\lambda^{-1}K_T}}.
\end{equation}
In particular, if the observed actions $x_t\in\X_t$ are optimal for $u$ (i.e., under \cref{ass:optimal-feedback}) and the boundedness assumption (\cref{ass:bounded-payoff}) holds, we have
\begin{equation}\label{eq:main-bound-B}
  R_T(u)\le B\log\det\prn*{\Id_T+\lambda^{-1}K_T} + \|u\|\sqrt{\lambda\log\det\prn*{\Id_T+\lambda^{-1}K_T}}.
\end{equation}
\end{theorem}
\begin{proof}
  For the first claim, we bound $R_T(u) =-\ip{u}{\zeta_T}$ from above as follows:
  \[
  -\ip{u}{\zeta_T}
  \le \normA{u}{A_T} \sqrt{\Phi_T}
  \overset{\hyperref[eq:cei]{\text{CEI}}}{\;\le\;} \normA{u} {A_T} \sqrt{\sum_{t=1}^T\ip{g_t}{A_t^{-1}g_t}} 
  \overset{\hyperref[eq:epl]{\text{EPL}}}{\;\le\;} \normA{u} {A_T} \sqrt{\log\det\prn*{\Id_T+\lambda^{-1}K_T}}.
  \]
  The first inequality uses the Cauchy--Schwarz inequality with respect to the $A_T$-inner product.
  The remaining two inequalities are shown later; 
  the second one, $\Phi_T \le \sum_{t=1}^T\ip{g_t}{A_t^{-1}g_t}$, is the \underline{c}umulative-potential--\underline{e}lliptical-potential \underline{i}nequality \eqref{eq:cei};
  the third, $\sum_{t=1}^T\ip{g_t}{A_t^{-1}g_t} \le \log\det\prn*{\Id_T+\lambda^{-1}K_T}$, is the well-known \underline{e}lliptical-\underline{p}otential \underline{l}emma \eqref{eq:epl}. 
  Assuming these, we obtain the first claim~\eqref{eq:main-bound-optimized}. 
  
  For the second claim, let $r_t = \ip{u}{x_t-\hat x_t}\in[0,B]$, where $0 \le r_t$ and $r_t \le B$ follow from the optimality of $x_t$ for $u$ (\cref{ass:optimal-feedback}) and the boundedness assumption (\cref{ass:bounded-payoff}), respectively. 
From $A_T = \lambda \Id + \sum_{t=1}^T g_t \otimes g_t$ and $\ip{u}{g_t}=-r_t$, we have
\[
  \normAsq{u}{A_T}
  =\lambda\norm{u}^2+\sum_{t=1}^T\ip{u}{g_t}^2
  =\lambda\norm{u}^2+\sum_{t=1}^T r_t^2
  \le \lambda\norm{u}^2+B\sum_{t=1}^T r_t
  = \lambda\norm{u}^2+BR_T(u).
\]
Combining this with \eqref{eq:main-bound-optimized}, with
$H_T\coloneqq\log\det\prn*{\Id_T+\lambda^{-1}K_T}$, gives
\[
  R_T(u)\le \sqrt{\lambda\norm{u}^2+BR_T(u)} \sqrt{H_T},
\]
equivalently $R_T(u)^2-BH_T R_T(u)-\lambda\norm{u}^2H_T\le 0$. 
Solving this quadratic inequality gives
\[
  R_T(u)\le \frac{1}{2} \prn*{BH_T+\sqrt{B^2H_T^2+4\lambda\norm{u}^2H_T}}
   \le BH_T + \|u\|\sqrt{\lambda H_T},
\]
where we used the subadditivity of the square root in the last step.
This establishes \eqref{eq:main-bound-B}.
\end{proof}

\subsubsection{Cumulative-potential--elliptical-potential inequality}\label{sec:sign}
We present the proof of the cumulative-potential--elliptical-potential inequality \eqref{eq:cei} used in the proof of \cref{thm:main}.
A similar argument is used in the analysis of the second-order perceptron \citep{Cesa-Bianchi2005-on}, but it focuses on the classification setting. 
First, we present a key sign condition.
\begin{lemma}[Sign condition]\label{lem:sign}
  For each iteration $t$ of \cref{alg:second-order}, we have $\ip{g_t}{A_{t-1}^{-1}\zeta_{t-1}}\le 0$.
\end{lemma}
\begin{proof}
The case of $t=1$ is trivial since $\zeta_0=0$. 
Let $t\ge 2$.
Recall $g_t=\hat x_t-x_t$ and $\hat w_t=-A_{t-1}^{-1}\zeta_{t-1}$.
Since $\hat x_t$ maximizes $x\mapsto\ip{\hat w_t}{x}$ over $\X_t$ and $x_t\in\X_t$, it holds that
$\ip{\hat w_t}{\hat x_t}\ge \ip{\hat w_t}{x_t}$, hence
$
  \ip{g_t}{A_{t-1}^{-1}\zeta_{t-1}}
  =  
  -\ip{\hat w_t}{g_t}
  =
  -\ip{\hat w_t}{\hat x_t-x_t}\le 0
$.
Thus, the sign condition holds for every $t$.
\end{proof}

This sign condition enables us to control the potential $\Phi_t$ in \eqref{eq:potential}, leading to the following proof of CEI.
\begin{lemma}[CP--EP inequality]\label{lem:potential-growth}
Let $A_t$ and $\zeta_t$ evolve as $A_t = A_{t-1} + g_t \otimes g_t$ and $\zeta_t = \zeta_{t-1} + g_t$ for $t = 1, \dots, T$, where $A_0 = \lambda \Id$ for $\lambda > 0$ and $\zeta_0 = 0$, as in \cref{alg:second-order}. 
Then, it holds that
\begin{equation}\tag{CEI}\label{eq:cei}
  \Phi_T\le \sum_{t=1}^T\ip{g_t}{A_t^{-1}g_t}.
\end{equation}
\end{lemma}
\begin{proof}
Fix $t$. 
By the Sherman--Morrison formula (see \cref{app:operator-preliminaries} for details), we have 
\begin{equation}\label{eq:smw}
  A_t^{-1}
  =
  A_{t-1}^{-1} - \frac{A_{t-1}^{-1}g_t \otimes A_{t-1}^{-1}g_t}{1 + \ip{g_t}{A_{t-1}^{-1}g_t}}.
\end{equation}
Defining $\mu_t\coloneqq\ip{g_t}{A_{t-1}^{-1}g_t}$ and $\nu_t\coloneqq\ip{g_t}{A_{t-1}^{-1}\zeta_{t-1}}$, and using the above identity, we obtain 
\[
\Phi_t-\Phi_{t-1}
 = \ip{\zeta_{t-1}+g_t}{ A_t^{-1}(\zeta_{t-1}+g_t)}-\ip{\zeta_{t-1}}{ A_{t-1}^{-1}\zeta_{t-1}}
 = 2\nu_t+\mu_t-\frac{(\nu_t + \mu_t)^2}{ 1+\mu_t}
 = \frac{\mu_t + 2\nu_t - \nu_t^2}{1+\mu_t}.
\]
The sign condition (\cref{lem:sign}) is exactly $\nu_t \le 0$, which implies $2\nu_t-\nu_t^2\le 0$ and hence 
$\Phi_t-\Phi_{t-1}\le {\mu_t}/{(1+\mu_t)}$.
Also, by using the Sherman--Morrison formula \eqref{eq:smw} again, we have
\begin{equation}\label{eq:gt-mu}
  \ip{g_t}{A_t^{-1}g_t}
  =\mu_t-\frac{\mu_t^2}{1+\mu_t}
  =\frac{\mu_t}{1+\mu_t}.
\end{equation}
Thus, $\Phi_t-\Phi_{t-1}\le \ip{g_t}{A_t^{-1}g_t}$ holds.
Summing over $t$ and using $\Phi_0 =0$ yield \eqref{eq:cei}.
\end{proof}

\subsubsection{Elliptical potential lemma}\label{sec:elliptical}
The remaining ingredient is the elliptical potential lemma \eqref{eq:epl}, which is a standard tool in the analysis of linear bandits~\citep{Dani2008-uo,Abbasi-yadkori2011-st} and ONS \citep{Hazan2007-ta}.
For completeness, we restate it here and present a proof adapted to our Hilbert-space setting in \cref{app:epl-logdet-proof}.
\begin{restatable}[Elliptical potential lemma]{lemma}{ellipticallemma}\label{lem:elliptical}
Let $\lambda>0$, $A_t=\lambda\Id+\sum_{s=1}^t g_s\otimes g_s$ for $t=1,\dots,T$, and $K_T = ( \ip{g_i}{g_j} )_{i,j=1}^T \in \R^{T\times T}$ as in \eqref{eq:gram-template}. Then, it holds that
\begin{equation}\tag{EPL}\label{eq:epl}
  \sum_{t=1}^T\ip{g_t}{A_t^{-1}g_t} \le \log\det\prn*{\Id_T+\lambda^{-1}K_T}.
\end{equation}
\end{restatable}
\wraprestatablewithlabelrestore{ellipticallemma}

\section{Robustness to suboptimal feedback actions}\label{sec:robust}
We now allow the feedback action $x_t$ to merely lie in the feasible set $\X_t$, not necessarily optimal for~$u$, and quantify how the main regret bound degrades with the cumulative suboptimality. 
The regret is still defined with respect to the revealed actions $x_t$, i.e., $R_T(u)=\sum_{t=1}^T\ip{u}{x_t-\hat x_t}$, even though~$x_t$ may be suboptimal. 
Here, the per-round regret $r_t=\ip{u}{x_t-\hat x_t}$ can be negative if $x_t$ is suboptimal for $u$, hindering the derivation of \eqref{eq:main-bound-B} in \cref{thm:main}. 
Still, $r_t \in [-B,B]$ holds by \cref{ass:bounded-payoff}.

We define the per-round suboptimality and its cumulative version as follows:
\begin{equation}\label{eq:subopt-loss}
  \delta_t(u)\coloneqq\max_{x\in\X_t}\ip{u}{x}-\ip{u}{x_t}\ge 0
  \qquad
  \text{and}
  \qquad
  \Delta_T(u)\coloneqq\sum_{t=1}^T\delta_t(u).
\end{equation}
The following lemma, essentially contained in the proof of \citet[Theorem~4.1]{Sakaue2025-vb}, is useful for handling suboptimal feedback. 
We give the proof in \cref{app:proof-self-bounding} for completeness.
\begin{restatable}[{cf.~\citep[Theorem~4.1]{Sakaue2025-vb}}]{lemma}{selfboundinglemma}\label{lem:self-bounding-subopt}
Under \cref{ass:bounded-payoff}, we have $\sum_{t=1}^T r_t^2\le B R_T(u)+2B\Delta_T(u)$.
\end{restatable}
\wraprestatablewithlabelrestore{selfboundinglemma}

Combining this with the template bound in \cref{thm:main} gives the following robust guarantee.
\begin{theorem}[Suboptimality-robust bound]\label{thm:robust-demo}
Assume only that $x_t\in\X_t$ is feasible (not necessarily optimal) and run \cref{alg:second-order} with $\lambda>0$.
Under \cref{ass:bounded-payoff}, for any $u\in\VV$, we have
\begin{equation}\label{eq:robust-demo-sqrt}
  R_T(u)\le BH_T+\norm{u}\sqrt{\lambda H_T}+\sqrt{2B\Delta_T(u)H_T},
\end{equation}
where we let $H_T = \log\det\prn*{\Id_T+\lambda^{-1}K_T}$, as in the proof of \cref{thm:main}. 
When $\Delta_T(u)=0$ (i.e., under \cref{ass:optimal-feedback}), this recovers \eqref{eq:main-bound-B} in \cref{thm:main}.
\end{theorem}
\begin{proof}
  If $R_T(u)\le 0$, \eqref{eq:robust-demo-sqrt} holds trivially as the right-hand side is nonnegative.
Assume $R_T(u)>0$.
From $A_T = \lambda\Id + \sum_{t=1}^T g_t \otimes g_t$, $r_t=\ip{u}{x_t - \hat x_t} = -\ip{u}{g_t}$, and \cref{lem:self-bounding-subopt}, we have
\begin{equation}\label{eq:robust-demo-norm}
  \normAsq{u}{A_T}=\lambda\norm{u}^2+\sum_{t=1}^T\ip{u}{g_t}^2=\lambda\norm{u}^2+\sum_{t=1}^T r_t^2 \le \lambda\norm{u}^2+B R_T(u)+2B\Delta_T(u).
\end{equation}
Note that the template bound \eqref{eq:main-bound-optimized} in \cref{thm:main} does not require optimality of $x_t$, and hence it remains valid here. 
Squaring both sides of \eqref{eq:main-bound-optimized} and applying \eqref{eq:robust-demo-norm} yield
\[
  R_T(u)^2 \le \|u\|_{A_T}^2 H_T \le H_T\prn*{\lambda\norm{u}^2+B R_T(u)+2B\Delta_T(u)}.
\]
Solving the quadratic inequality and using the subadditivity of the square root give the bound \eqref{eq:robust-demo-sqrt}.
\end{proof}

\textbf{Comparison with the prior approach.\;}
\Citet[Theorem~4.1]{Sakaue2025-vb} obtained a similar robust guarantee by running $\Theta(\log T)$ ONS instances, based on MetaGrad \citep{van-Erven2021-ji}; this increases the computation cost by a factor of $\log T$.
In contrast, our CoRectron itself is suboptimality-robust without changing the algorithm.
This is enabled by our algorithm design and analysis, which make explicit the role of the data-dependent norm $\normA{u}{A_T}$ arising from the second-order update with the cumulative residuals.

\section{Application to contextual models}\label{sec:lifting}
This section explains how various contextual settings fit into our Hilbert-space formulation. 
Similar ideas for finite-dimensional settings are discussed in previous studies \citep{Besbes2021-ak,Besbes2025-ck,Sakaue2025-vb}, but the kernelized setting, described below, has not been explicitly covered. 
In offline inverse optimization, kernelized formulations have served as a useful approach for handling non-linearity \citep{Bertsimas2015-kw,Long2024-gl}; 
therefore, we believe that the applicability of our online method to the kernelized setting is a meaningful contribution.


\subsection{Generic lifting template}\label{sec:template}
Let $n \in \mathbb{Z}_{>0}$.
Suppose each round gives a base action set $\X_t^{\mathrm{base}}\subseteq\R^n$ and contextual information, encoded as a bounded linear map $\Psi_t\colon\R^n\to\VV$.
We lift $\X_t^{\mathrm{base}}$ to $\X_t$ as the image of $\X_t^{\mathrm{base}}$ under $\Psi_t$, i.e.,
\begin{equation}\label{eq:lifted-set} 
  \X_t\coloneqq\Psi_t(\X_t^{\mathrm{base}})=\{\Psi_t(x):x\in \X_t^{\mathrm{base}}\}\subseteq\VV.
\end{equation}
Let $\Psi_t^\ast\colon\VV\to\R^n$ be the adjoint, given by $\ip{w}{\Psi_t(x)}_{\VV}=\ip{\Psi_t^\ast w}{x}_{\R^n}$ for all $w\in\VV$ and $x\in\R^n$.\footnote{
  We use subscripts to distinguish the inner products in $\VV$ and $\R^n$.
}
Then, maximizing $\ip{w}{\cdot}_{\VV}$ over $\X_t$ is equivalent to maximizing $\ip{\Psi_t^\ast w}{\cdot}_{\R^n}$ over $\X_t^{\mathrm{base}}$.
In particular, a linear-optimization oracle over $\X_t^{\mathrm{base}}$ implements the oracle over $\X_t$ required by \cref{alg:second-order}.

\textbf{Base-action feedback and residuals.\;}
In contextual settings, the interaction is naturally phrased in the base action space: the learner observes $\X_t^{\mathrm{base}}$ and recommends an action in $\X_t^{\mathrm{base}}$, after which a user action $x_t^{\mathrm{base}}\in\X_t^{\mathrm{base}}$ is revealed.
Given a learner's prediction $\hat w_t\in\VV$, let 
$\hat w_t^{\mathrm{base}}\coloneqq\Psi_t^\ast\hat w_t$ and
\[
  \hat x_t^{\mathrm{base}}\in\argmax
  \Set*{\ip{\hat w_t^{\mathrm{base}}}{x}_{\R^n}}{x\in\X_t^{\mathrm{base}}}.
\]
Then, $\hat x_t=\Psi_t(\hat x_t^{\mathrm{base}})$ belongs to $\argmax_{x\in\X_t}\ip{\hat w_t}{x}_{\VV}$. 
Moreover, the revealed base action $x_t^{\mathrm{base}}$ induces the revealed lifted action $x_t=\Psi_t(x_t^{\mathrm{base}})\in\X_t$, and hence we obtain the residual as
\[
  g_t=\hat x_t-x_t=\Psi_t(g_t^{\mathrm{base}})\in\VV,
  \qquad\text{where}\qquad
  g_t^{\mathrm{base}}\coloneqq\hat x_t^{\mathrm{base}}-x_t^{\mathrm{base}}\in\R^n,  
\]
thanks to the linearity of $\Psi_t$.
Therefore, base-action feedback is sufficient to compute the residuals required by \cref{alg:second-order}.
In what follows, we assume the diameter of the base-action sets $\X_t^{\mathrm{base}}$ is uniformly bounded by some $X>0$ in the Euclidean norm, so that $\|g_t^{\mathrm{base}}\|_2\le X$ holds for all $t$.

\subsection{Examples}\label{sec:examples}

\textbf{Non-contextual model.\;}
Take $\VV=\R^n$ and $\Psi_t=\Id_n$. Then, we have $\X_t=\X_t^{\mathrm{base}}$ and $\Psi_t^\ast=\Id_n$.

\textbf{Linear contextual model.\;}
Suppose that contexts are given as $z_t\in\R^p$ and consider linear utilities $\X^\mathrm{base}_t\ni
x\mapsto \ip{W z_t}{x}_{\R^n}$ parameterized by $W\in\R^{n\times p}$.
Take $\VV=\R^{n\times p}$ with the Frobenius inner product $\ip{A}{B}_\mathrm{F}=\tr(A^\top B)$ and define $\Psi_t(x)= x z_t^\top$.
Then, the adjoint is $\Psi_t^\ast(W)=W z_t$ since we have
$\ip{W}{\Psi_t(x)}_\mathrm{F}=\tr(W^\top x z_t^\top)=\ip{W z_t}{x}_{\R^n}$.

\textbf{Kernelized contextual model.\;}
Let $\Z$ be a context space and $\VV=\HH$ an RKHS of vector-valued functions $h\colon\Z\to\R^n$ with kernel $\mathcal{K}$.
For each context $z\in\Z$, let $\mathcal{K}_z\colon\R^n\to\HH$ be the canonical feature operator $(\mathcal{K}_z x)(\cdot)= \mathcal{K}(\cdot,z)x$ for $x\in\R^n$.
By the reproducing property, we have
\begin{equation}\label{eq:reproducing}
  \ip{h}{\mathcal{K}_z x}_{\HH}=\ip{h(z)}{x}_{\R^n},
  \qquad
  \forall h\in\HH,\ \forall x\in\R^n.
\end{equation}
Given a context $z_t \in \Z$ at round $t$, let $\Psi_t= \mathcal{K}_{z_t}$. 
Then, the adjoint is $\Psi_t^\ast(h)=h(z_t)$ by \eqref{eq:reproducing}. 

\subsection{Instantiating regret bounds}\label{sec:specializations}
Our regret bounds, \eqref{eq:main-bound-B} and \eqref{eq:robust-demo-sqrt}, depend on the structure of the action space $\X_t$ through the log-determinant potential $H_T = \log\det\prn*{\Id_T+\tfrac{1}{\lambda}K_T}$, where $K_T$ is the Gram matrix \eqref{eq:gram-template} of residuals~$g_t$.
Below, we bound $H_T$ in the three models of \cref{sec:examples}.
For later use, we define the \emph{effective dimension} of $K_T$ for $\lambda>0$ as follows, which is commonly used in kernelized online learning \citep{Calandriello2017-cu,Calandriello2017-tj}:
\[
  d_{\mathrm{eff}}^T(\lambda)\coloneqq \tr\prn*{K_T(K_T+\lambda \Id_T)^{-1}} = \sum_{i=1}^m \frac{\sigma_i}{\sigma_i+\lambda},
\]
where $\sigma_1,\dots,\sigma_m > 0$ are the non-zero eigenvalues of $K_T$.
We have $d_{\mathrm{eff}}^T(\lambda)\le \rank(K_T)$; in finite-dimensional models this is at most $\dim(\VV)$. 
In infinite-dimensional models, $d_{\mathrm{eff}}^T(\lambda)$ can be finite even if $\dim(\VV)$ is infinite. 
The following lemma upper bounds $H_T$ in terms of the effective dimension.
\begin{restatable}{lemma}{logdetdefflemma}\label{lem:logdet-deff}
For any positive semidefinite matrix $K_T \in \R^{T\times T}$ and $\lambda>0$, we have
\[
  \log\det\prn*{\Id+\lambda^{-1}K_T}
  \le d_{\mathrm{eff}}^T(\lambda)\prn*{1+\log\prn*{1+\lambda^{-1}\|K_T\|_{\mathrm{op}}}},
\]
where $\|K_T\|_{\mathrm{op}}$ is the operator norm of $K_T$.
\end{restatable}
\wraprestatablewithlabelrestore{logdetdefflemma}
See \cref{app:logdet-deff-proof} for the proof of \cref{lem:logdet-deff}.
Using this lemma in our regret bounds, we obtain the following specific regret bounds for the three models. 
We present the detailed derivations in \cref{app:modelwise-ht-bounds}.
\begin{restatable}[Instantiated regret bounds]{corollary}{instantiatedregretcorollary}\label{cor:logT}
  Under \cref{ass:bounded-payoff}, 
  assuming $\norm{x - x'}_2\le X$ for all $x,x'\in\X^\mathrm{base}_t$ and all $t$ for some $X>0$, we have the following regret bounds.
\begin{enumerate}[label=(\roman*), leftmargin=*, nosep]
\item \textbf{Non-contextual.\;}
Run \cref{alg:second-order} with $\lambda=B^2 n$.
For any $u\in\R^n$ with $\norm{u}_2\le1$, we have
  \[
\textstyle
		    \sum_{t=1}^T \ip{u}{x_t - \hat x_t} = O\prn[\Big]{B n\log\prn*{2+\tfrac{T X}{B n}}+\raisebox{-0.0ex}{$\sqrt{{\raisebox{0.0ex}{$B\Delta_T(u)\,n\log\prn*{2+\tfrac{T X}{B n}}$}}}$}}.
  \]
\item \textbf{Linear contextual.\;}
Assume additionally $\|z_t\|_2\le Z$ for all $t$ for some $Z > 0$.
Run \cref{alg:second-order} with $\lambda=B^2 np$.
Then, for any $U \in \R^{n\times p}$ with $\|U\|_\mathrm{F}\le 1$, we have
\[\textstyle
		    \sum_{t=1}^T \ip{U z_t}{x_t - \hat x_t} = O\prn[\Big]{B np\log\prn[\big]{2+\tfrac{T XZ}{B np}}+\raisebox{-0.0ex}{$\sqrt{{\raisebox{0.0ex}{$B\Delta_T(U)\,np\log\prn[\big]{2+\tfrac{T XZ}{B np}}$}}}$}}.
\]
\item \textbf{Kernelized contextual.\;}
Assume additionally $\|\mathcal{K}(z,z)\|_{\mathrm{op}}\le \kappa^2$ for all $z\in\Z$ for some $\kappa > 0$.
Run \cref{alg:second-order} with an arbitrary $\lambda\in(0,T]$.
Then, for any $h^\star\in\HH$ with $\norm{h^\star}_\HH\le 1$,
we have
\[\textstyle
		    \sum_{t=1}^T \ip{h^\star(z_t)}{ x_t - \hat x_t} = O\prn[\Big]{ B d_{\mathrm{eff}}^T(\lambda) \log\prn*{2 + \tfrac{T X\kappa}{\lambda}} + \sqrt{ {\prn*{\lambda +  B\Delta_T(h^\star)} d_{\mathrm{eff}}^T(\lambda) \log\prn*{2 + \tfrac{T X\kappa}{\lambda}}}}}.
\]
\end{enumerate}
In particular, under optimal feedback (\cref{ass:optimal-feedback}), the above bounds hold with $\Delta_T(\cdot)=0$.
\end{restatable}
\wraprestatablewithlabelrestore{instantiatedregretcorollary}
\pagebreak

The non-contextual result recovers the state-of-the-art $O(n \log T)$ regret bound of the existing studies \citealt[Theorem~4.2]{Gollapudi2021-ad} \citealt[Theorem~3.1]{Sakaue2025-vb}, as well as its suboptimality-robust version 
\citep[Theorem~4.1]{Sakaue2025-vb}.\footnote{
  As in the prior work, we state the bounds with $\lambda$ tuned under the assumption that $B$ is known, for direct comparison. For implementation, however, it suffices to use any known upper bound on $B$, or one may simply take $\lambda=n$ and $\lambda=np$ in the non-contextual and linear contextual cases, respectively, which changes the argument of the logarithm by a multiplicative factor of $B$.

} 
By contrast, the extension to the infinite-dimensional kernelized setting is a new result of this work.

\section{Implementation and computational considerations}\label{sec:implementation}

We discuss how to implement \cref{alg:second-order} and the dominant per-round computational factors.
In contextual models, the update proceeds as follows: 
solve the linear system for $\hat w_{t}=-A_{t-1}^{-1}\zeta_{t-1}$, 
compute $\hat w_t^{\mathrm{base}}=\Psi_t^\ast\hat w_t$, solve linear optimization to obtain $\hat x_t^{\mathrm{base}}\in\argmax_{x\in\X_t^{\mathrm{base}}}\ip{\hat w_t^{\mathrm{base}}}{x}$, compute $g_t=\Psi_t(\hat x_t^{\mathrm{base}}-x_t^{\mathrm{base}})$, 
and 
update $A_t = A_{t-1}+g_t\otimes g_t$ and $\zeta_t=\zeta_{t-1}+g_t$.
Among these steps, major computational factors are 
(i) solving the linear system for $\hat w_{t}$, 
(ii) feature map evaluations, 
and 
(iii) linear optimization to find~$\hat x_t^{\mathrm{base}}$.
Writing the computational costs of these steps at round~$t$ as 
$\tau_{\mathrm{solve}}(t)$, 
$\tau_{\mathrm{feat}}(t)$, 
and 
$\tau_{\mathrm{opt}}(t)$, respectively, we express the per-round computational cost as
\begin{equation}\label{eq:impl-cost-template}
  \tau_{\mathrm{solve}}(t)
  +
  \tau_{\mathrm{feat}}(t)
  +
  \tau_{\mathrm{opt}}(t).
\end{equation}

\textbf{Finite-dimensional models.\;}
Consider finite-dimensional cases with $\VV=\R^d$, where $d=n$ in the non-contextual model and $d=np$ in the linear contextual model.
Evaluating feature maps for computing $g_t$ takes $\tau_{\mathrm{feat}}(t)=O(d)$. 
With the rank-one-update structure of $A_t=A_{t-1}+g_t \otimes g_t$, solving the linear system for $\hat w_{t+1}$ can be done in $O(d^2)$ time, hence $\tau_{\mathrm{solve}}(t)=O(d^2)$.
Therefore, the per-round time complexity is 
$
\tau_{\mathrm{feat}}(t)+\tau_{\mathrm{opt}}(t)+\tau_{\mathrm{solve}}(t)
  =
  O\prn*{d^2+\tau_{\mathrm{opt}}(t)}
$, 
where~$\tau_{\mathrm{opt}}(t)$ depends on the structure of $\X_t^{\mathrm{base}}$. 
Summing over $t$ yields the total time complexity in the bottom row of \cref{tab:finite-dim-cost-comparison}.

\textbf{Kernelized contextual model.\;}
We present the implementation details for the kernelized model in \cref{app:kernelized-impl-details}, and here we summarize the main computational factors.
Based on the representer form, we work with the Gram system $K_t+\lambda \Id_t$.
Letting $\tau_{Kv}$ denote the cost of one kernel-vector product $\mathcal{K}(z,z')v\in\R^n$, we can implement the update so that we have 
$\tau_{\mathrm{solve}}(t)=O(t^2)$
and 
$\tau_{\mathrm{feat}}(t)=O(t\,\tau_{Kv})$. 
Hence, the per-round cost is 
$\tau_{\mathrm{solve}}(t)
+
\tau_{\mathrm{feat}}(t)+\tau_{\mathrm{opt}}(t)
  =
  O\prn*{t^2+t\,\tau_{Kv}+\tau_{\mathrm{opt}}(t)}$.

\section{Conclusion}\label{sec:conclusion}
We have presented an efficient projection-free algorithm with logarithmic regret bounds for contextual recommendation.
The same algorithm is also shown to be robust to suboptimal feedback actions.
Our analysis reveals the power of the second-order-perceptron update for contextual recommendation, which we believe is a meaningful conceptual contribution.
We have also provided a unified view for handling linear and kernelized contextual models.
Experiments in \cref{app:dashboard-layout-preview} demonstrate the empirical advantages of our algorithm over prior approaches.

Regarding limitations, our guarantees still rely on access to an exact linear-optimization oracle over each feasible set.
In addition, the exact kernelized implementation still incurs $O(t^2)$ per-round cost, and thus further approximation or compression is needed for truly large-scale or very long-horizon settings.
Addressing these limitations is an important and interesting direction for future work.

\section*{Acknowledgements}
Shinsaku Sakaue is supported by JST BOOST Program Grant Number JPMJBY24D1.

{\newrefcontext[sorting=nyt]
\printbibliography[heading=bibintoc]}

\clearpage
\appendix
\section{Additional preliminaries on operators}\label{app:operator-preliminaries}
%
%
%

Most of the fundamental tools used in this paper are covered in \cref{sec:setup}. 
Here, we present additional preliminaries on operators that are used in the proofs.

In the proof of \cref{lem:potential-growth}, we use the Sherman--Morrison identity of the form 
\[
  \prn*{A+g\otimes g}^{-1}
  =
  A^{-1}
  -
  \frac{A^{-1}g\otimes A^{-1}g}{1+\ip{g}{A^{-1}g}},
\]
where $A\colon\VV\to\VV$ is an invertible operator and $g$ is a vector in $\VV$. 
Also, in the proof of \cref{lem:elliptical} in \cref{app:epl-logdet-proof}, we use the Woodbury identity of the form
\[
  \prn*{\lambda\Id+GG^\ast}^{-1}
  =
  \lambda^{-1}\Id-\lambda^{-2}G\prn*{\Id + \lambda^{-1} G^\ast G}^{-1}G^\ast,
\]
where $G\colon\R^m\to\VV$ is a bounded linear map, $G^\ast\colon\VV\to\R^m$ is its adjoint, and $\lambda>0$.
These are special cases of the following Sherman--Morrison--Woodbury identity. 
While the finite-dimensional version is well-known, it extends to the general Hilbert-space setting (e.g., \citet[Theorem~1.1]{Deng2011-co}).
We present the statement and proof here for completeness.
\begin{proposition}
Let $A\colon\VV\to\VV$ be an invertible operator and $G\colon\R^m\to\VV$ be a bounded linear map.
Assume that $\Id_m+G^\ast A^{-1} G$ is invertible.
Then, we have
\[
  (A + GG^\ast)^{-1} = A^{-1} - A^{-1}G(G^\ast A^{-1} G + \Id_m)^{-1}G^\ast A^{-1}.
\]
\end{proposition}

\begin{proof}
Let $M\coloneqq G^\ast A^{-1}G\in\R^{m\times m}$ and assume $\Id_m+M$ is invertible.
Define
\[
  B\coloneqq(\Id_m+M)^{-1}
  \qquad\text{and}\qquad
  X\coloneqq A^{-1}-A^{-1}GBG^\ast A^{-1}.
\]
We first check that $X$ is a left inverse of $A+GG^\ast$:
\begin{align}
  (A+GG^\ast)X
  &= \Id + GG^\ast A^{-1} - GBG^\ast A^{-1} - GG^\ast A^{-1}GBG^\ast A^{-1}
  \\
  &= \Id + G(\Id_m-B-MB)G^\ast A^{-1}
  = \Id,  
\end{align}
where the last equality follows from $(\Id_m+M)B=\Id_m$.
Likewise, $X$ is a right inverse:
\begin{align}
  X(A+GG^\ast)
  &= \Id + A^{-1}GG^\ast - A^{-1}GBG^\ast - A^{-1}GBG^\ast A^{-1}GG^\ast
  \\
  &= \Id + A^{-1}G(\Id_m-B-BM)G^\ast
  = \Id,
\end{align}
where we used $B(\Id_m+M)=\Id_m$.
Hence, $X=(A+GG^\ast)^{-1}$, i.e.,
\[
  (A + GG^\ast)^{-1} = A^{-1} - A^{-1}G(G^\ast A^{-1} G + \Id_m)^{-1}G^\ast A^{-1},
\]
as claimed.
\end{proof}

\section{Proof of Lemma~\ref{lem:elliptical}}\label{app:epl-logdet-proof}
\ellipticallemma*
\begin{proof}
As shown in \eqref{eq:gt-mu}, $\ip{g_t}{A_t^{-1}g_t}=\frac{\mu_t}{1+\mu_t}$ holds, where $\mu_t\coloneqq\ip{g_t}{A_{t-1}^{-1}g_t}$.
Using ${\mu}/{(1+\mu)}\le \log(1+\mu)$ for all $\mu\ge 0$ and summing over $t$, we obtain
\begin{equation}\label{eq:log-mu}
  \sum_{t=1}^T\ip{g_t}{A_t^{-1}g_t}
  \le \sum_{t=1}^T\log(1+\mu_t).
\end{equation}
It remains to show
\begin{equation}\label{eq:logdet-mu}
  \sum_{t=1}^T\log(1+\mu_t)=\log\det\prn*{\Id_T+\lambda^{-1}K_T}.
\end{equation}
For each $t$, define
\[
  K_t=(\ip{g_i}{g_j})_{i,j=1}^t\in\R^{t\times t},
  \qquad
  k_t=(\ip{g_i}{g_t})_{i=1}^{t-1}\in\R^{t-1},
  \qquad\text{and}\qquad
  M_t=\Id_t+\lambda^{-1}K_t.
\]
Let $G_{t-1}\colon\R^{t-1}\to\VV$ be given by $G_{t-1}e_i=g_i$ for $i=1,\dots,t-1$, where $e_i$ is the $i$-th standard basis vector in $\R^{t-1}$.
Its adjoint $G_{t-1}^\ast\colon\VV\to\R^{t-1}$ is defined by $\ip{G_{t-1}a}{v}=a^\top(G_{t-1}^\ast v)$ for all $a\in\R^{t-1}$ and $v\in\VV$, so that $(G_{t-1}^\ast v)_i=\ip{g_i}{v}$ holds.
Thus, for all $v\in\VV$, we have
\[
  G_{t-1}G_{t-1}^\ast v = \sum_{s=1}^{t-1} \ip{g_s}{v} g_s = \sum_{s=1}^{t-1} (g_s \otimes g_s) v,
\] 
hence $G_{t-1}G_{t-1}^\ast=\sum_{s=1}^{t-1} g_s\otimes g_s$.
Also, for all $a\in\R^{t-1}$, we have
\[
  G_{t-1}^\ast G_{t-1} a = \sum_{s=1}^{t-1} a_s G_{t-1}^\ast g_s = \sum_{s=1}^{t-1} a_s (\ip{g_1}{g_s},\dots,\ip{g_{t-1}}{g_s})^\top = K_{t-1} a,
\]
hence $G_{t-1}^\ast G_{t-1}=K_{t-1}$.
Therefore, we have 
\[
  A_{t-1}=\lambda\Id+G_{t-1}G_{t-1}^\ast \quad\text{and}\quad M_{t-1}=\Id_{t-1}+\lambda^{-1}G_{t-1}^\ast G_{t-1}.
\]
The Woodbury identity (see \cref{app:operator-preliminaries}) yields
\[
  A_{t-1}^{-1}=\lambda^{-1}\Id-\lambda^{-2}G_{t-1}M_{t-1}^{-1}G_{t-1}^\ast.
\]
Hence, we have
\begin{equation}\label{eq:A-woodbury}
  \mu_t
  =\ip{g_t}{A_{t-1}^{-1}g_t}
  =\lambda^{-1}\ip{g_t}{g_t}-\lambda^{-2}k_t^\top M_{t-1}^{-1}k_t.
\end{equation}
Moreover, $M_t$ has the block decomposition:
\[
  M_t=
  \begin{pmatrix}
    M_{t-1} & \lambda^{-1}k_t\\
    \lambda^{-1}k_t^\top & 1+\lambda^{-1}\ip{g_t}{g_t}
  \end{pmatrix}.
\]
By the Schur complement formula and \eqref{eq:A-woodbury}, we have
\[
  \det(M_t)=\det(M_{t-1})(1+\mu_t).
\]
Iterating this identity (with $\det(M_0)=1$) gives
$\det(M_T)=\prod_{t=1}^T(1+\mu_t)$, which is equivalent to the desired equality \eqref{eq:logdet-mu}.
Combining \eqref{eq:log-mu} and \eqref{eq:logdet-mu} proves \eqref{eq:epl}.
\end{proof}

\section{Proof of Lemma~\ref{lem:self-bounding-subopt}}
\label{app:proof-self-bounding}
\selfboundinglemma*
\begin{proof}
Fix $t$ and note that $|r_t|\le B$ holds by \cref{ass:bounded-payoff}.
Since $\hat x_t\in\X_t$ is feasible, the definition of~$\delta_t(u) \ge 0$ gives
\begin{equation}\label{eq:delta-r}
  \delta_t(u)
  =
  \max_{x\in\X_t}\ip{u}{x}-\ip{u}{x_t}
  \ge \ip{u}{\hat x_t}-\ip{u}{x_t}=-r_t.
\end{equation}
If $r_t\ge 0$, then $0 \le r_t \le B$ implies $r_t^2 - Br_t \le 0$.
If $r_t<0$, then $0 \le -r_t \le B$ and \eqref{eq:delta-r} imply
\[
  r_t^2 - Br_t
  \le
  B(-r_t) - Br_t
  \le
  2B\delta_t.
\]
Thus, in any case, we have 
$
r_t^2 - Br_t
  \le
  \max\{0, 2B\delta_t\}
  =
  2B\delta_t
$.
Summing over $t$ gives $\sum_{t=1}^T r_t^2\le B R_T(u)+2B\Delta_T(u)$, as desired.
\end{proof}

\section{Proof of Lemma~\ref{lem:logdet-deff}}\label{app:logdet-deff-proof}
\logdetdefflemma*
\begin{proof}
Let $\sigma_1,\dots,\sigma_m > 0$ be the non-zero eigenvalues of $K_T$.
Then, we have
\[
  \log\det\prn*{\Id+\lambda^{-1}K_T}
  =\sum_{i=1}^m\log\prn*{1+\frac{\sigma_i}{\lambda}}.
\]
Since $\log(1+x)\le x$ for $x\ge 0$, we have
\[
(1+x)\log(1+x)
=
\log(1+x) + x\log(1+x)
\le 
x(1+\log(1+x)),
\]
that is,
\[
\log(1+x)\le \frac{x}{1+x}\prn*{1+\log(1+x)}.
\]
Applying this with $x=\sigma_i/\lambda$ and using $\log(1+\sigma_i/\lambda)\le \log(1+\|K_T\|_{\mathrm{op}}/\lambda)$ yield
\[
  \log\prn*{1+\frac{\sigma_i}{\lambda}}
  \le \frac{\sigma_i}{\sigma_i+\lambda}\prn*{1+\log\prn*{1+\lambda^{-1}\|K_T\|_{\mathrm{op}}}}.
\]
Summing over $i$ and using $\sum_{i=1}^m \frac{\sigma_i}{\sigma_i+\lambda}=d_{\mathrm{eff}}^T(\lambda)$ prove the claim.
\end{proof}

\section{Proof of Corollary~\ref{cor:logT}}\label{app:modelwise-ht-bounds}
\instantiatedregretcorollary*
\begin{proof}[Proof of Corollary~\ref{cor:logT}]
Given \cref{lem:logdet-deff}, it remains to bound $\|K_T\|_{\mathrm{op}}$ for the three models (and $d_{\mathrm{eff}}^T(\lambda)$ for the finite-dimensional models).
From $K_T = (\ip{g_i}{g_j})_{i,j=1}^T$, $g_t=\Psi_t(g_t^{\mathrm{base}})$, and $\norm{g^\mathrm{base}_t}_2\le X$, we have the following bound on $\|K_T\|_{\mathrm{op}}$, which we will use later:
\[
  \|K_T\|_{\mathrm{op}}\le \tr(K_T)
  =\sum_{t=1}^T\ip{g_t^{\mathrm{base}}}{\Psi_t^\ast \Psi_t(g_t^{\mathrm{base}})}_{\R^n}
  \le \sum_{t=1}^T\|\Psi_t^\ast \Psi_t\|_{\mathrm{op}}\|g_t^{\mathrm{base}}\|_2^2
  \le X^2\sum_{t=1}^T\|\Psi_t^\ast \Psi_t\|_{\mathrm{op}}.
\]

\textbf{Non-contextual model ($\VV=\R^n$).\;}
From $\Psi_t = \Psi_t^\ast = \Id_n$, we have $\|K_T\|_{\mathrm{op}}\le T X^2$. Also, we have
$d_{\mathrm{eff}}^T(\lambda)\le \dim(\VV)=n$.
Therefore, we obtain
\[
  H_T\le n\prn*{1+\log\prn*{1+\tfrac{T X^2}{\lambda}}}.
\]

\textbf{Linear contextual model ($\VV=\R^{n\times p}$).\;}
Assume additionally $\|z_t\|_2\le Z$ for all $t$.
Then, we have $\|\Psi_t^\ast \Psi_t\|_{\mathrm{op}}=\|z_t z_t^\top\|_{\mathrm{op}}=\|z_t\|_2^2\le Z^2$, hence $\|K_T\|_{\mathrm{op}}\le T X^2 Z^2$. 
Also, $d_{\mathrm{eff}}^T(\lambda)\le \dim(\VV)=np$ holds.
Therefore, we obtain
\[
  H_T\le np\prn*{1+\log\prn*{1+\tfrac{T X^2 Z^2}{\lambda}}}.
\]

\textbf{Kernelized contextual model (RKHS).\;}
Assume additionally $\|\mathcal{K}(z,z)\|_{\mathrm{op}}\le \kappa^2$ for all $z\in\Z$.
Then, $\|\Psi_t^\ast \Psi_t\|_{\mathrm{op}}=\|\mathcal{K}(z_t,z_t)\|_{\mathrm{op}}\le \kappa^2$ holds, hence $\|K_T\|_{\mathrm{op}}\le T X^2 \kappa^2$.
Therefore, we obtain
\[
  H_T\le d_{\mathrm{eff}}^T(\lambda)\prn*{1+\log\prn*{1+\tfrac{T X^2 \kappa^2}{\lambda}}}.
\]

Below, we simplify those bounds on $H_T$. Note that the bounds take the form
$H_T \le d(1+\log(1+a))$ for some positive $d$ and $a$. For $a \ge 0$, it holds that
$1+\log(1+a) \le \frac{1}{\log 2}\log(2+a)$. Therefore, we have
$H_T = O(d\log(2+a))$. Let us take a closer look at the argument $a$ in each model, with $\log(2 + bc^2)\le2\log(2+bc)$ for $b \ge 1$ and $c > 0$ in mind.
In the non-contextual case, suppose $T > n$; otherwise $R_T(u) \le Bn$ trivially holds.
Then, setting $\lambda = B^2 n$ yields
$H_T = O\left(n \log\left(2+\frac{TX}{Bn}\right)\right)$.
Similarly, in the linear contextual case, we can assume $T > np$ (otherwise $R_T(u) \le Bnp$ trivially holds), and setting $\lambda = B^2 np$ yields
$H_T = O\left(np \log\left(2+\frac{TXZ}{Bnp}\right)\right)$.
In the kernelized case, we keep $\lambda \in (0,T]$ arbitrary, hence
$H_T = O\left(d_T^{\mathrm{eff}}(\lambda)\log\left(2+\frac{TX\kappa}{\lambda}\right)\right)$.
Combining these bounds with~\eqref{eq:robust-demo-sqrt} in \cref{thm:robust-demo} and using the condition that
$u$, $U$, and $h^\star$ have norm at most $1$ in their respective spaces yield the
displayed regret bounds.\looseness=-1
\end{proof}


\section{Kernelized implementation details}\label{app:kernelized-impl-details}

We present the implementation details for the kernelized model discussed in \cref{sec:implementation}. 
At the beginning of every round $t$, we maintain
\[
  L_{t-1}L_{t-1}^\top = K_{t-1}+\lambda \Id_{t-1}
  \qquad
  \text{and}
  \qquad
  c_{t-1}=\prn*{K_{t-1}+\lambda \Id_{t-1}}^{-1}\mathbf{1}_{t-1},
\]
where $L_{t-1}\in\R^{(t-1)\times (t-1)}$ is the Cholesky factor of $K_{t-1}+\lambda \Id_{t-1}$ and $\mathbf{1}_{t-1}\in\R^{t-1}$ is the all-ones vector.
Define $G_{t-1}\colon\R^{t-1}\to\HH$ by $G_{t-1}e_s=g_s$, so that we have
\[
  \zeta_{t-1}=G_{t-1}\mathbf{1}_{t-1}, \quad
  A_{t-1}=\lambda \Id+G_{t-1}G_{t-1}^\ast, \quad
  \text{and}
  \quad
  K_{t-1}=G_{t-1}^\ast G_{t-1}.
\]
Using $\hat w_t=-A_{t-1}^{-1}\zeta_{t-1}$ and 
$\prn*{\lambda\Id+GG^\ast}^{-1}G
  =
  G\prn*{G^\ast G+\lambda\Id}^{-1}$, 
  which holds for any bounded linear map $G$ and $\lambda>0$, we have
\[
  \hat w_t
  =
  -\prn*{\lambda\Id+G_{t-1}G_{t-1}^\ast}^{-1}G_{t-1}\mathbf{1}_{t-1}
  =
  -G_{t-1}\prn*{K_{t-1}+\lambda \Id_{t-1}}^{-1}\mathbf{1}_{t-1}
  =
  -\sum_{s=1}^{t-1}(c_{t-1})_s g_s.
\]
Using these maintained quantities, we proceed as follows 
(for convenience, we put the computation of $\hat w_{t+1}$ at the end of round $t$, but it is equivalent to computing $\hat w_{t}$ at the beginning of round $t$):
\begin{enumerate}[leftmargin=2em, itemsep=0.5em]
\item Compute the predicted objective vector
\[
  \hat w_t^{\mathrm{base}}=\Psi_t^\ast \hat w_t=-\sum_{s=1}^{t-1}(c_{t-1})_s\,\Psi_t^\ast g_s.
\]
In the kernelized model, we have $\Psi_t^\ast g_s=\mathcal{K}(z_t,z_s)g_s^{\mathrm{base}}$, and thus this computation requires $t-1$ kernel-vector products.
\item Call the base linear optimization oracle to compute
\[
  \hat x_t^{\mathrm{base}}\in\argmax \Set*{\ip{\hat w_t^{\mathrm{base}}}{x}}{{x\in\X_t^{\mathrm{base}}}},
\]
observe $x_t^{\mathrm{base}}$, and set $g_t=\Psi_t\prn*{\hat x_t^{\mathrm{base}}-x_t^{\mathrm{base}}}$.
\item Compute the new Gram entries
\[
  k_t=\prn*{\ip{g_s}{g_t}_{\HH}}_{s=1}^{t-1}
  \qquad
  \text{and}
  \qquad
  \rho_t=\ip{g_t}{g_t}_{\HH},
\]
where $\ip{g_s}{g_t}_{\HH}
  =
  \ip{g_s^{\mathrm{base}}}{\mathcal{K}(z_s,z_t)g_t^{\mathrm{base}}}_{\R^n}$.
\item Update the Cholesky factor from $L_{t-1}$ to $L_t$ for $K_t+\lambda \Id_t$ using the block update
\[
  K_t+\lambda \Id_t
  =
  \begin{bmatrix}
    K_{t-1}+\lambda \Id_{t-1} & k_t \\
    k_t^\top & \rho_t+\lambda
  \end{bmatrix}.
\]
Specifically, given $L_{t-1}L_{t-1}^\top=K_{t-1}+\lambda \Id_{t-1}$, we first solve
\[
  L_{t-1}y_t=k_t.
\]
Then we set
\[
  \beta_t=\sqrt{\rho_t+\lambda-\|y_t\|_2^2}
  \qquad
    \text{and}
    \qquad
  L_t=
  \begin{bmatrix}
    L_{t-1} & 0 \\
    y_t^\top & \beta_t
  \end{bmatrix}.
\]
Since $K_t+\lambda \Id_t\succ 0$, $\beta_t$ is well-defined, and $L_tL_t^\top=K_t+\lambda \Id_t$ holds.
\item Solve
\[
  \prn*{K_t+\lambda \Id_t}c_t=\mathbf{1}_t
\]
by forward substitution with $L_t$ and backward substitution with $L_t^\top$, and set
\[
  \hat w_{t+1}=-\sum_{s=1}^{t}(c_t)_s g_s.
\]
\end{enumerate}
In each round, computing $\hat w_t^{\mathrm{base}}$ and updating the new Gram row/column each require $O(t\,\tau_{Kv})$ time. Thus, the feature-evaluation cost is
\[
  \tau_{\mathrm{feat}}(t)=O\prn*{t\,\tau_{Kv}}.
\]
From Steps 4--5, solving the linear system for $\hat w_{t+1}$ consists of (i) one forward substitution for $L_{t-1}y_t=k_t$ and
(ii) one forward and one backward substitution for
$\prn*{K_t+\lambda \Id_t}c_t=\mathbf{1}_t$ via $L_t$ and~$L_t^\top$.
Thus, the cost of solving the linear system is
\[
  \tau_{\mathrm{solve}}(t)=O\prn*{(t-1)^2}+O\prn*{t^2}=O(t^2).
\]
Therefore, the total per-round computational cost is
\[
  \tau_{\mathrm{solve}}(t) + \tau_{\mathrm{feat}}(t)+\tau_{\mathrm{opt}}(t)
  =
  O\prn*{t^2+t\,\tau_{Kv}+\tau_{\mathrm{opt}}(t)}.
\]


\section{Comparison with ONS-based methods}
\label{app:ons-comparison}

We discuss the time complexity bounds of the ONS-based methods shown in \cref{tab:finite-dim-cost-comparison} in the finite-dimensional setting, i.e., $\VV=\R^d$.

The ONS-based method of \citet{Sakaue2025-vb} shares the same $\tau_{\mathrm{feat}}(t)$ and $\tau_{\mathrm{opt}}(t)$ terms, but each round additionally solves
$
\hat w_t\in\argmin_{w: \|w\|\le 1}\norm{w'_t-w}_{A_{t-1}}
$
for Mahalanobis projection, where~$w'_t$ is the unconstrained point.
The fastest implementation (e.g., \citet[Algorithm~4]{Wang2025-lj}) runs in~$O(d^\omega)$ time (up to logarithmic factors), yielding the per-round time complexity of $O\prn*{d^\omega+\tau_{\mathrm{opt}}(t)}$. 
Summing this over $t$ gives the total time complexity in the top row of \cref{tab:finite-dim-cost-comparison}.
Since the current best matrix multiplication exponent is $\omega \simeq 2.3714$~\citep{Alman2025-df} (and standard implementations typically take~$O(d^3)$ time), this is more expensive than our CoRectron. 

LightONS, proposed by \citet{Wang2025-lj}, is an efficient variant of ONS.
It achieves a lower total time complexity of $O\prn*{Td^2+d^\omega\sqrt{T \log T}}$ by reducing the number of projection steps, while the~$\tau_{\mathrm{opt}}(t)$ part remains the same as $\hat x_t$ is needed to construct a subgradient.
Thus, using LightONS instead of the standard ONS in \citet{Sakaue2025-vb} yields the total time complexity in the middle row of \cref{tab:finite-dim-cost-comparison}.

When it comes to kernelized settings, the analogous projection subproblem typically takes $O(t^3)$ and is therefore even more expensive than the per-round runtime of CoRectron. 
This practical advantage is also reflected in the experiments in \cref{app:dashboard-layout-preview}, where CoRectron typically achieves lower runtime than ONS while maintaining competitive or smaller regret.\looseness=-1 

\section{Numerical experiments}
\label{app:dashboard-layout-preview}

We conducted the experiments on a MacBook Air equipped with an Apple M4 chip (10 cores) and 32\,GB RAM, running macOS 26.2.
The reported results were generated using Python 3.11.14 with NumPy 2.4.3, pandas 3.0.1, and Matplotlib~3.10.8.

\subsection{Experiments with optimal-action feedback}\label[appendix]{subsec:experiments-optimal}
This section presents experiments with optimal-action feedback; 
\cref{subsec:experiments-suboptimal} considers suboptimal-action feedback.

\subsubsection{Experimental setup}

\textbf{Problem setting: contextual $m$-out-of-$n$.\;}
At each round $t$, a context $z_t\in\R^{p}$ is observed, and we define the (unnormalized) action set
\[
  \X_{\mathrm{raw}}
  =
  \set*{x\in\set*{0,1}^{n}\colon \sum_{i=1}^{n}x_i=m}.
\]
This corresponds to selecting $m$ items from $n$ candidates.
In our experiments, we fix $n=10$, $m=5$, and $p=10$, and the user and learner choose actions from the normalized action set
\[
  \X_t^{\mathrm{base}}
  =
  \frac{1}{\sqrt{10}}\,\X_{\mathrm{raw}}.
\]
That is, each action vector is scaled by $1/\sqrt{10}$, thereby ensuring that any $x,x'\in\X_t^{\mathrm{base}}$ satisfies $\norm{x-x'}_2\le 1$. 
We use the same base action set $\X_t^{\mathrm{base}}$ for all $t$ for simplicity, but the contexts~$z_t$ vary over time; 
therefore, the lifted action sets $\X_t$ vary over time.
Let $u^\mathrm{base}_t\in\R^{n}$ denote the user's base utility vector (specified later for linear and kernel models), and let
\[
  x_t^{\mathrm{base}} \in \argmax\Set*{\ip{u^\mathrm{base}_t}{x}}{x\in\X_t^{\mathrm{base}}}
  \quad
  \text{and}
  \quad 
  R_T = \sum_{t=1}^{T} \ip{u^\mathrm{base}_t}{x_t^{\mathrm{base}}-\hat x_t^{\mathrm{base}}}.
\]
Here, $x_t^{\mathrm{base}}$ is the user's selected base action, 
$\hat x_t^{\mathrm{base}}$ is the learner's recommended base action, 
and~$R_T$ is the cumulative regret.

\textbf{Data generation.\;} 
At each round, we sample
      $\tilde z_t\sim\mathcal{N}\prn*{0,\Id_p}$ (standard normal) and set
      \[
        z_t=
        \begin{cases}
          \tilde z_t & \norm{\tilde z_t}_2\le 1\\
          \tilde z_t/\norm{\tilde z_t}_2 & \norm{\tilde z_t}_2>1
        \end{cases}
      \]
      so that $\norm{z_t}_2\le 1$ always holds.
To generate the true utility $u^\mathrm{base}_t$, we consider two settings, referred to below as Linear and Kernel, where Kernel uses the radial basis function (RBF) kernel.
\begin{itemize}[leftmargin=*, itemsep=0.2em, topsep=0.3em]
\item Linear:
      $u^\mathrm{base}_t = U^\star z_t$,
      where 
      $U^\star\in\R^{n\times p}$, and
      $\norm{U^\star}_{\mathrm{F}}=1$
      ($\norm{\cdot}_{\mathrm{F}}$ is the Frobenius norm).
\item Kernel:
      $u^\mathrm{base}_t = h^\star\prn*{z_t}$, where
      \[
        h^\star\prn*{z}=\sum_{j=1}^{J}k_{\mathrm{rbf}}\prn*{z,c_j;\vartheta_{\mathrm{rbf}}}\,a_j
        \qquad
        \text{and}
        \qquad
        k_{\mathrm{rbf}}\prn*{z,z';\vartheta_{\mathrm{rbf}}}
        =
        \exp\prn*{-\frac{\norm{z-z'}_2^2}{2\vartheta_{\mathrm{rbf}}^2}}.
      \]
      Here, $\vartheta_{\mathrm{rbf}}>0$ is the RBF bandwidth (length scale),
      $c_j\in\R^{p}$ are kernel centers, $a_j\in\R^{n}$ are coefficient vectors,
      and $J=16$ is the number of centers.
The centers $c_j$ are independently drawn from the standard normal distribution
and then projected onto the unit ball, followed by rescaling to match the context support.
      Coefficients are first sampled from the standard normal distribution, $a_j^{\mathrm{raw}}\sim\mathcal{N}\prn*{0,\Id_n}$, and we set
      $A_{\mathrm{raw}}=[\prn*{a_1^{\mathrm{raw}}}^\top;\dots;\prn*{a_J^{\mathrm{raw}}}^\top]\in\R^{J \times n}$ 
      and
      $K_{ij}= k_{\mathrm{rbf}}\prn*{c_i,c_j;\vartheta_{\mathrm{rbf}}}$. 
      We then rescale the coefficients as
      \[
        A=\frac{1}{\sqrt{\tr\prn*{A_{\mathrm{raw}}^\top K A_{\mathrm{raw}}}}}A_{\mathrm{raw}},
      \]
      thereby normalizing the RKHS norm of $h^\star$ to $1$.
\end{itemize}

\textbf{Evaluation scenarios.\;}
In the Linear setting, we set the time horizon to $T=10000$.
For the Kernel setting, we evaluate 
$\vartheta_{\mathrm{rbf}}\in\{0.5,1.0,2.0\}$ 
and set $T=1000$ to save computational cost.
For each scenario, we run experiments with 10 random seeds; in each seed, all algorithms share the same round sequence (i.e., the same $z_t,u^\mathrm{base}_t$). 
Results are averaged over the 10 seeds. 

\textbf{Compared methods.\;}
We compare CoRectron-L (CoRectron for the linear-contextual setting),
ONS (Online Newton Step), OGD (Online Gradient Descent),
CoRectron-K (CoRectron for the RBF-kernel setting), and
KONS (Kernelized ONS).
In the Linear scenario, we evaluate the first three methods;
in the Kernel scenario, we evaluate all five methods.

Below, we describe the methods other than CoRectron; CoRectron itself was introduced in the main text. 
For convenience, we adopt the vectorized form.
All methods are initialized at the origin and iteratively perform updates using
the base-action residual gradient
$g_t^{\mathrm{base}} = \hat x_t^{\mathrm{base}}-x_t^{\mathrm{base}}$
at each round.
\begin{itemize}[leftmargin=*, itemsep=0.2em, topsep=0.3em]
\item OGD (Linear):
      Let $g_t= \mathrm{vec}(g_t^{\mathrm{base}} z_t^\top)\in\R^{np}$, where $\mathrm{vec}(\cdot)$ is the vectorization operator.
      With the step size $\eta_\mathrm{OGD} > 0$, it performs the update as
      \[
        \hat w_{t+1}
        =
        \argmin\ \Set*{
          \norm{w-\prn*{\hat w_t-\eta_\mathrm{OGD} g_t}}_2
        }{
          w\in\R^{np},\;\norm{w}_2\le 1
        }.
      \]
\item ONS (Linear):
      We use the same $g_t$.
      Following \citet{Sakaue2025-vb}, we apply ONS to exp-concave surrogate losses.
      In this case, the surrogate-loss parameter $\eta_{\mathrm{sur}}>0$ is fixed, and the ONS gradient is written as
      $\tilde g_t = \eta_{\mathrm{sur}} g_t$.
      Let $A_0=\epsilon \Id_{np}$ and perform the following update:
      \[
        A_t = A_{t-1}+\tilde g_t\tilde g_t^\top
        \qquad
        \text{and}
        \qquad
        \hat y_{t+1}=\hat w_t-\gamma_{\mathrm{ons}}^{-1}A_t^{-1}\tilde g_t.
      \]
      We then compute the Mahalanobis projection of $\hat y_{t+1}$ onto the unit ball:
      \[
        \hat w_{t+1}
        =
        \argmin\ \Set*{
          \norm{w-\hat y_{t+1}}_{A_t}
        }{
          w\in\R^{np},\;\norm{w}_2\le 1
        }.
      \]
\item KONS (Kernel):
      In the RKHS, we use the coefficient representation
      $\hat w_t=\sum_{i=1}^{t-1}\hat c_i\phi_i$, where
      $\phi_i= \mathcal{K}_{z_i}g_i^{\mathrm{base}}$.
      Define the zero-padded coefficient vector
      \[
        \hat c_t^{\mathrm{ext}}= \prn*{\hat c_1,\dots,\hat c_{t-1},0}^\top\in\R^t.
      \]
      Let $G_t$ be the Gram matrix of $\phi_1,\dots,\phi_t$, define
      $\tilde G_t\coloneqq \eta_{\mathrm{sur}}^2 G_t$, and let
      $e_t\in\R^t$ be the basis vector whose last component is $1$.
      We also define
\[
  B_t \coloneqq \tilde G_t+\epsilon \Id_t.
\]
      Then, the update is written as
\[
  q_t=B_t^{-1}e_t
  \qquad
    \text{and}
    \qquad
  \hat y_{t+1}=\hat c_t^{\mathrm{ext}}-\gamma_{\mathrm{ons}}^{-1}\prn*{\eta_{\mathrm{sur}}q_t}. 
\]
We then compute the Mahalanobis projection of $\hat y_{t+1}$ in the coefficient space:
\[
  \hat c_{t+1}
  =
  \argmin\ \Set*{
    \norm{c-\hat y_{t+1}}_{B_t}
  }{
    c\in\R^t,\;\norm*{\sum_{i=1}^{t}c_i\phi_i}_{\HH}\le 1
  }.
\]
\end{itemize}

\textbf{Hyperparameter setting.\;} 
Recall that the diameter of the base action set is $1$, and the utility vectors are normalized so that $\norm{u^\mathrm{base}_t}_2 \le 1$. 
Therefore, we set the bound $B$ that appears in \cref{ass:bounded-payoff}~to~$1$.
OGD, ONS, and KONS, described above, use projection onto the unit ball, and we set the radius parameter to $1$ for these methods. 
Given these choices, the analysis of \citet{Sakaue2025-vb} suggests using the surrogate-loss scaling parameter $\eta_{\mathrm{sur}}=1/10$ and the common step-size parameter $\gamma_{\mathrm{ons}}=1/2$ for both ONS and KONS. 
From $n = p = 10$, we define $d= np = 100$ in the Linear setting. 
In the Kernel setting, the effective dimension that appears in the theory is not available in advance. 
Thus, we use the same value $d = 100$ as its proxy. 
The remaining hyperparameters are $\eta_\mathrm{OGD}$ for OGD, $\epsilon$ for ONS and KONS, and $\lambda$ for CoRectron-L and CoRectron-K.
These act as regularization-strength parameters: larger $\epsilon$ or $\lambda$, and smaller $\eta_\mathrm{OGD}$, make the update more conservative. 
Because these parameters often strongly affect empirical performance, we introduce reference values $\bar\eta_\mathrm{OGD}$, $\bar\epsilon$, and $\bar\lambda$, and sweep them over the common coefficient grid 
$c\in\{10^{-3},10^{-2},10^{-1},1,10,10^2,10^3\}$ 
as follows:

\begin{figure}[tb]
  \centering
  \includegraphics[
    width=.7\textwidth
  ]{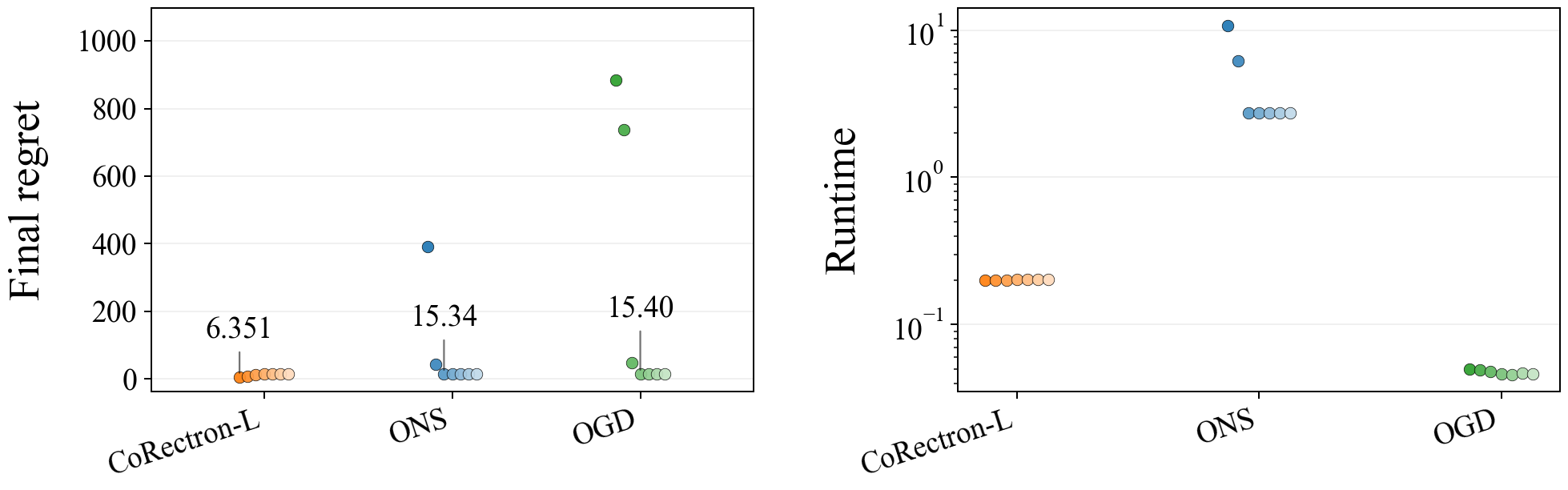}
  \caption{Overall comparison in the Linear setting. Final regret and runtime over the coefficient grid are shown. All values are means over 10 seeds. In each algorithm block, darker points correspond to smaller $c$, and $c$ increases from left to right. For final regret, the annotated value is the best over the coefficient sweep.}
  \label{fig:spread-dashboard-linear}
\end{figure}

\begin{figure}[tb]
  \centering
  \includegraphics[
    width=.7\textwidth
  ]{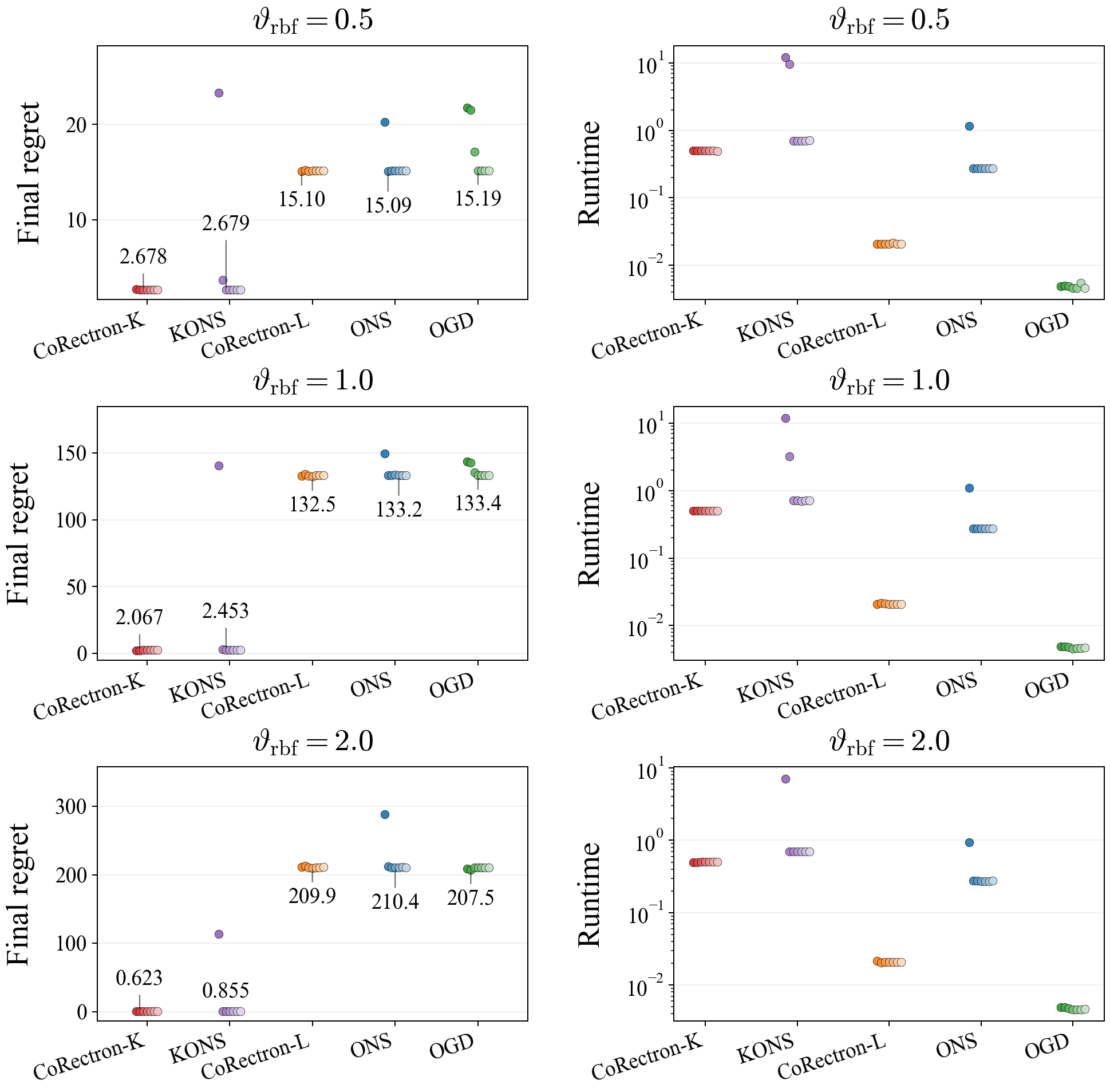}
  \caption{Overall comparison in the Kernel-RBF setting. Final regret and runtime over the coefficient grid are shown. All values are means over 10 seeds. In each algorithm block, darker points correspond to smaller $c$, and $c$ increases from left to right. For final regret, the annotated value is the best over the coefficient sweep.}
  \label{fig:spread-dashboard-kernel-rbf}
\end{figure}

\begin{figure}[tb]
  \centering
  \includegraphics[
    width=.7\textwidth
  ]{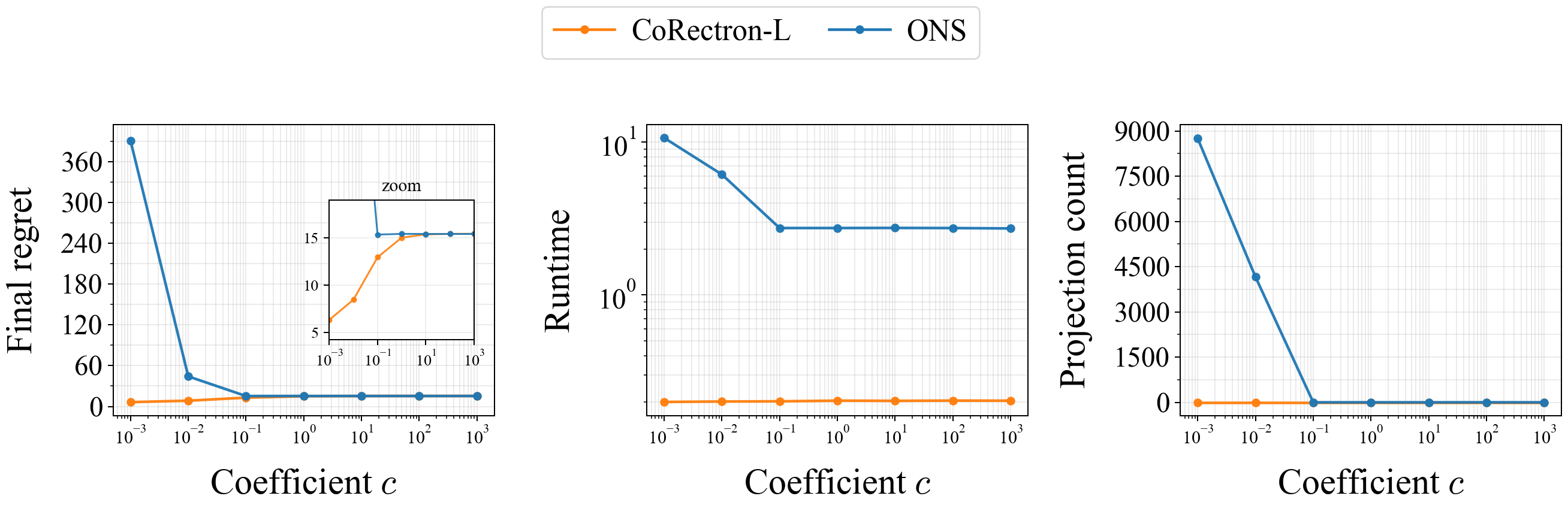}
  \caption{Comparison of CoRectron-L and ONS in the Linear setting. Final regret, runtime, and total number of Mahalanobis projections across coefficient values $c$ are shown. Shaded bands indicate the~95\% confidence intervals over 10 seeds, computed by using Student's $t$ distribution, namely mean $\pm t_{0.975,n-1}\, s/\sqrt{n_\text{seeds}}$, where $s$ is the sample standard deviation and $n_\text{seeds}$ is the number of seeds.
}
  \label{fig:curves-dashboard-linear}
\end{figure}

\begin{figure}[tb]
  \centering
  \includegraphics[
    width=.7\textwidth
  ]{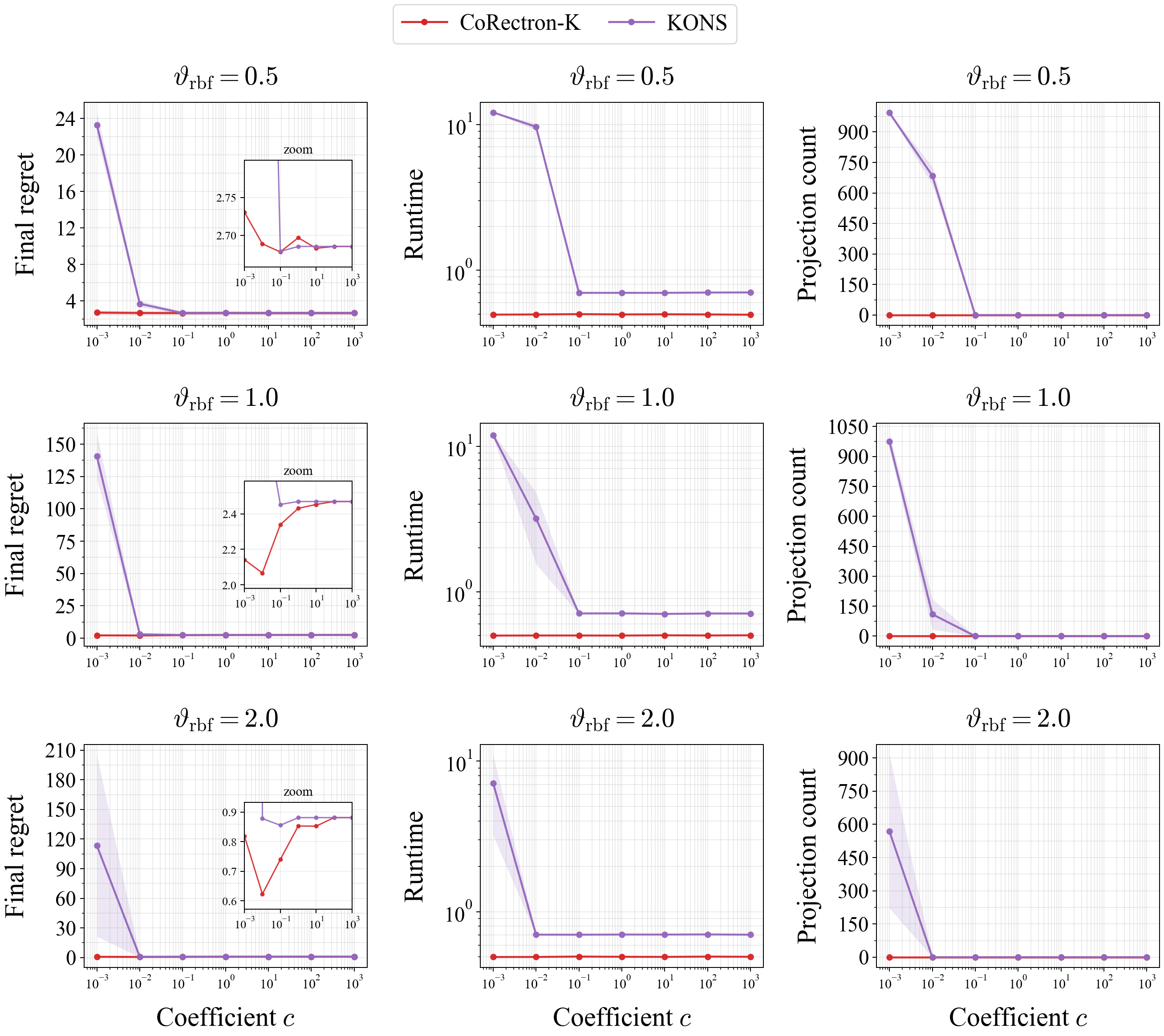}
  \caption{Comparison of CoRectron-K and KONS in the Kernel-RBF setting. Final regret, runtime, and total number of Mahalanobis projections across coefficient values $c$ are shown. Shaded bands indicate the~95\% confidence intervals over 10 seeds, computed as with \Cref{fig:curves-dashboard-linear}.
}
  \label{fig:curves-dashboard-kernel-rbf}
\end{figure}

\begin{figure}[tb]
  \centering
  \includegraphics[width=.45\textwidth]{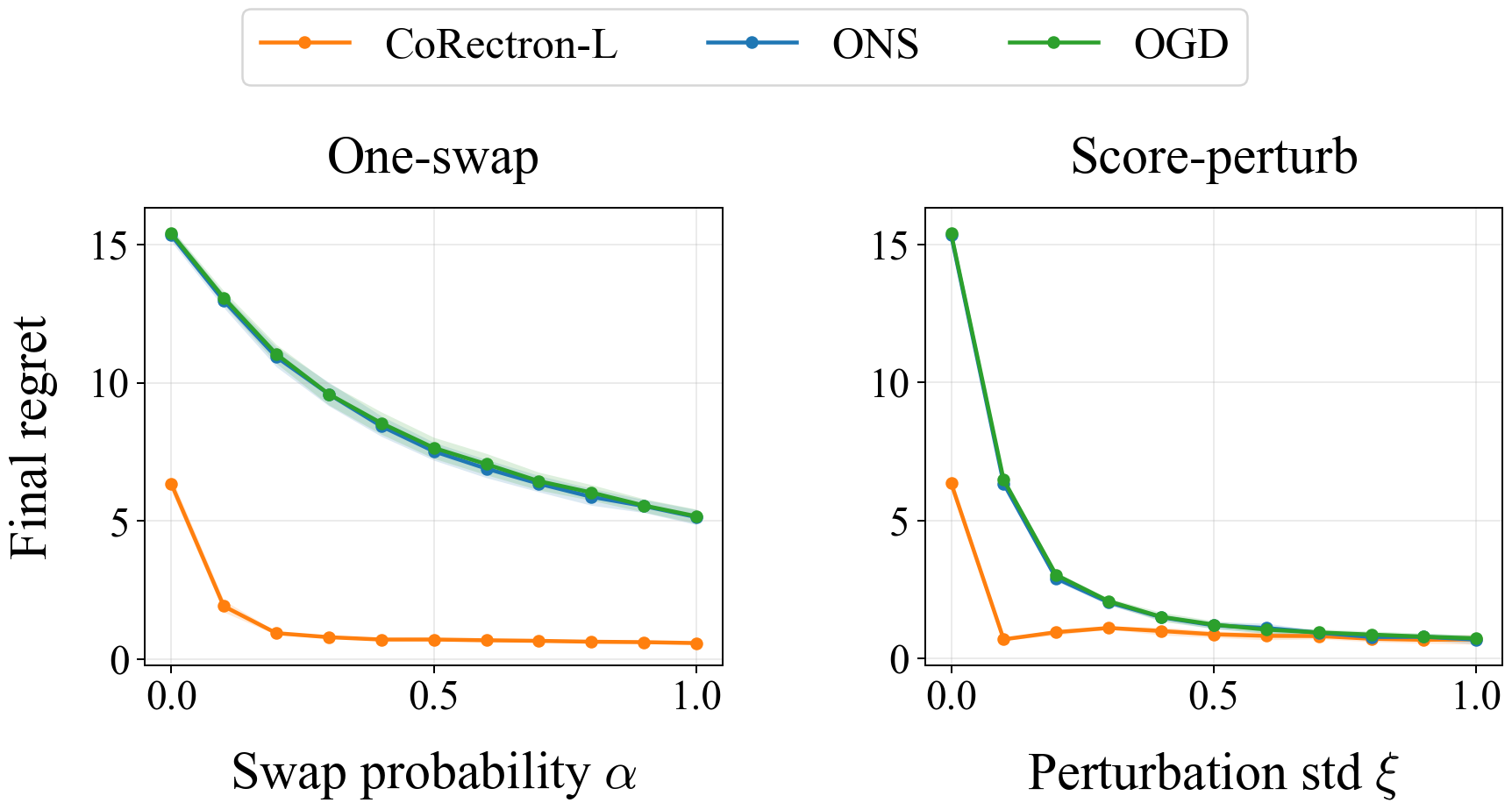}
  \caption{%
    Robustness to suboptimal-action feedback in the Linear setting.
    Left: final regret $R_T$ vs.\ swap probability $\alpha$ (one-swap model).
    Right: final regret $R_T$ vs.\ relative noise level $\xi$ (score-perturb model).
    Each method uses the best-performing coefficient $c$ from \cref{subsec:experiments-optimal}.
    All values are means over 10 seeds; shaded bands indicate 95\% confidence intervals
    computed as with \Cref{fig:curves-dashboard-linear}.
  }
  \label{fig:subopt-combined}
\end{figure}

\begin{itemize}[leftmargin=*, itemsep=0.2em, topsep=0.3em]
\item OGD: As the domain radius and the gradient norm are at most $1$, we set the reference step size to $\bar\eta_\mathrm{OGD} = \tfrac{2}{\sqrt{T}}$ and sweep $\eta_\mathrm{OGD}\leftarrow \bar\eta_\mathrm{OGD}/c$. Therefore, larger $c$ corresponds to a smaller step size.
\item ONS and KONS: As the domain radius is $1$, and the dimension and the step-size coefficient are $d=100$ and $\gamma_{\mathrm{ons}}=1/2$, respectively, we set the reference regularization strength to $\bar\epsilon= \tfrac{d}{4\gamma_{\mathrm{ons}}^2}=100$ based on \citet{Sakaue2025-vb}, and sweep $\epsilon\leftarrow c \bar\epsilon$. 
\item CoRectron-L and -K: As $B=1$ and $d=100$, we set the reference regularization to $\bar\lambda= B^2 d = d = 100$ and sweep $\lambda\leftarrow c\,\bar\lambda$.
\end{itemize}

\begin{remark}[On LightONS]
  As discussed in \cref{sec:related-work} and \cref{app:ons-comparison}, we can use LightONS~\citep{Wang2025-lj} instead of the standard ONS to reduce the number of Mahalanobis projection steps, thereby improving the runtime.
  In our experiments, however, we observe that the standard ONS with appropriately tuned regularization strength $\epsilon$ achieves its best regret without performing any projection, and thus LightONS does not yield a significant improvement in this case. 
  Therefore, we report the results with the standard ONS, where those with best-performing $\epsilon$ serve as a rough proxy for the performance of LightONS.
\end{remark}

\subsubsection{Results}
We report the following metrics, all averaged over the 10 seeds:
\begin{itemize}[leftmargin=*, itemsep=0.2em, topsep=0.3em]
\item Final regret: the cumulative regret $R_T$ at the end of the horizon $T$.
\item Total runtime: wall-clock time in seconds, measured by \texttt{time.perf\_counter}. To standardize runtime comparisons, we fix the number of threads of BLAS/OpenMP-related libraries to 1.
\item Total number of Mahalanobis projections: 
For ONS and KONS, we monitor the total number of rounds in which nontrivial Mahalanobis projections are performed, thereby observing how much they affect the runtime. 
We do not count trivial projections where the input is already within the unit ball.
For CoRectron and OGD, this count is identically zero as they do not perform such projections.
\end{itemize}

\textbf{Overall comparison.\;}
\Cref{fig:spread-dashboard-linear,fig:spread-dashboard-kernel-rbf}
show the final regret and total runtime of each method over the seven coefficient values in the sweep,
reported separately for the Linear and Kernel settings as averages over 10 random seeds. 
The darker points correspond to smaller $c$, and $c$ increases from left to right within each algorithm block.
For final regret, the best value within the sweep over the scaling coefficient $c$ is also annotated in the figures.

We first examine \Cref{fig:spread-dashboard-linear}, which leads to the following observations:
\begin{itemize}[leftmargin=*, itemsep=0.2em, topsep=0.3em]
\item CoRectron-L with $c = 10^{-3}$ achieves the best final regret.
      The best values of ONS and OGD are roughly tied, and both are more than twice as large as the best value of CoRectron-L.
\item CoRectron-L exhibits low final regret and stable runtime throughout the sweep, indicating stability to hyperparameter tuning.
      In contrast, for ONS and OGD, making the regularization too weak
      (that is, taking $c$ too small) causes both final regret and runtime to increase substantially.
      Moreover, the best final regret of these methods is attained at the smallest value of $c$ that avoids such deterioration, suggesting that they require relatively delicate hyperparameter tuning.
\item In terms of runtime, OGD is the fastest, while CoRectron-L is faster than ONS across all $c$ values.
\end{itemize}

We next turn to \Cref{fig:spread-dashboard-kernel-rbf}.
For reference, this figure also plots the results of the methods designed for the Linear setting,
namely CoRectron-L, ONS, and OGD.
The following observations can be made:
\begin{itemize}[leftmargin=*, itemsep=0.2em, topsep=0.3em]
\item CoRectron-K achieves the best final regret, slightly outperforming KONS.
      In addition, the final regret of all methods designed for the Linear setting is substantially worse
      than that of the methods designed for the Kernel setting,
      demonstrating the importance of using a contextual model that matches the underlying data generation.
\item As in the Linear case, KONS also exhibits deteriorated final regret and runtime when $c$ is too small.
      In contrast, CoRectron-K, like CoRectron-L in the Linear setting, achieves consistently low final regret and runtime
      across the tuning of the regularization parameter.
\end{itemize}

\textbf{Comparisons of CoRectron-L/-K and ONS/KONS.\;}
\Cref{fig:curves-dashboard-linear,fig:curves-dashboard-kernel-rbf}
provide a closer comparison between CoRectron-L/-K and ONS/KONS.
The horizontal axis is the coefficient $c$, and the vertical axes report
final regret, runtime, and the total number of Mahalanobis projections.
The shaded bands indicate the 95\% confidence intervals around the means over 10 seeds.
\begin{itemize}[leftmargin=*, itemsep=0.2em, topsep=0.3em]
\item Consistent with the overall comparison above, CoRectron achieves the best final regret, maintaining high performance and efficiency stably across changes in $c$.
\item The runtime curves of ONS/KONS track their projection-count curves,
      indicating that the frequency of Mahalanobis projections significantly affects the runtime of these methods.
      Accordingly, a major reason CoRectron runs more stably and faster than ONS/KONS is that it requires no Mahalanobis projection. This highlights the practical importance of eliminating Mahalanobis projection steps.
\end{itemize}

\subsection{Experiments with suboptimal-action feedback}\label[appendix]{subsec:experiments-suboptimal}
This section examines the robustness of the compared algorithms when the observed user action can be suboptimal, as discussed in \cref{sec:robust}.
We use the Linear setting of \cref{subsec:experiments-optimal} throughout;
refer to \cref{subsec:experiments-optimal} for the problem setting, data generation,
and descriptions of the compared methods (CoRectron-L, ONS, OGD).

\subsubsection{Experimental setup}

\textbf{Hyperparameter setting.\;}
We fix each method's regularization coefficient at the best-performing value identified
in the sweep of \cref{subsec:experiments-optimal}:
$c=10^{-3}$ for CoRectron-L, $c=10^{-1}$ for ONS, and $c=1$ for OGD.

\textbf{Suboptimal-action models.\;}
We consider two models of possibly suboptimal user actions.
In both models, the true utility vector $u_t^{\mathrm{base}}$ is generated as in
\cref{subsec:experiments-optimal}.
\begin{itemize}[leftmargin=*, itemsep=0.2em, topsep=0.3em]
\item One-swap (swap probability $\alpha\in[0,1]$):
      With probability $\alpha$, the item ranked $m$-th under $u_t^{\mathrm{base}}$
      (the marginal item in the optimal selection) is replaced by the item ranked $(m+1)$-th
      (the best excluded item); with probability $1-\alpha$, the optimal action
      $x_t^{\mathrm{base}}\in\argmax_{x\in\X_t^{\mathrm{base}}}\ip{u_t^{\mathrm{base}}}{x}$
      is revealed unchanged.
      This yields per-round suboptimality $\delta_t = u_t^{\mathrm{base}}(m) - u_t^{\mathrm{base}}(m+1) \ge 0$
      whenever the swap occurs, where $u_t^{\mathrm{base}}(k)$ denotes the $k$-th largest component.
      At $\alpha=0$ this coincides with the optimal-feedback setting.
      We sweep $\alpha\in\{0.0,0.1,\ldots,1.0\}$.
\item Score-perturb (relative noise level $\xi\ge 0$):
      The revealed action is taken as the argmax of the perturbed utility
      $u_t^{\mathrm{base}}+\xi\norm{u_t^{\mathrm{base}}}_2\,\varepsilon_t$,
      where $\varepsilon_t\sim\mathcal{N}(0,\Id_n)$ is drawn independently each round.
      Scaling the noise magnitude by $\norm{u_t^{\mathrm{base}}}_2$ makes $\xi$ a
      relative noise level that is invariant to the scale of~$u_t^{\mathrm{base}}$.
      At $\xi=0$ this recovers the optimal-feedback setting.
      We sweep $\xi\in\{0.0,0.1,\ldots,1.0\}$.
\end{itemize}

\begin{remark}[On MetaGrad-based method]
  We leave an experimental comparison with the MetaGrad-based method \citep{van-Erven2016-mg,Sakaue2025-vb} to future work, since implementing and tuning that baseline would require substantial additional engineering effort. It also needs to maintain $\Theta(\log T)$ ONS instances, leading to considerably higher computational cost than the methods compared here in practice.
\end{remark}

\subsubsection{Results}

\Cref{fig:subopt-combined} shows the final regret as a function of $\alpha$ and $\xi$.
The following observations can be made:
\begin{itemize}[leftmargin=*, itemsep=0.2em, topsep=0.3em]
\item CoRectron-L achieves the lowest final regret across all noise levels in both settings. The relative ordering among CoRectron-L, ONS, and OGD is preserved throughout the noise sweep in both models, with CoRectron-L maintaining a consistent advantage.
\item The final regret of all methods tends to decrease as $\alpha$ or $\xi$ increases. This reflects the definition of regret $R_T = \sum_{t=1}^{T} \ip{u^\mathrm{base}_t}{x_t^{\mathrm{base}}-\hat x_t^{\mathrm{base}}}$: as $x_t^{\mathrm{base}}$ becomes more suboptimal, the target the learner must compete against weakens, reducing $R_T$. This trend is not directly reflected by the worst-case regret upper bound in \cref{thm:robust-demo} involving the $\sqrt{2B\Delta_T(u)H_T}$ term, which is increasing in the cumulative suboptimality $\Delta_T(u)$, suggesting that higher suboptimality should worsen performance. Bridging this gap---developing a theoretical guarantee that explains the empirical decrease in regret as the user action becomes more suboptimal---is an interesting open problem.
\end{itemize}

\end{document}